\documentclass[twoside,11pt]{article}

\usepackage[nohyperref]{jmlr2e}

\usepackage{algorithmic,algorithm}

\usepackage{microtype}

\usepackage[colorlinks,citecolor=blue,urlcolor=blue,linkcolor=blue,linktocpage=true]{hyperref}
\usepackage{url}

\usepackage{multirow,makecell}
\usepackage{mathtools}
\usepackage{bbm}
\usepackage{enumerate}
\usepackage{subfigure}
\usepackage{color}

\newcommand{\argmin}{\mathop{\mathrm{argmin}}}
\newcommand{\argmax}{\mathop{\mathrm{argmax}}}

\newcommand{\blue}[1]{\textcolor{blue}{#1}}

\def\E{\mathbb{E}}
\def\P{\mathbb{P}}
\def\1{\mathbbm{1}}
\def\Dis{\mathtt{Dis}}

\def\cD{\mathcal{D}}

\def\cH{\mathcal{H}}

\def\cM{\mathcal{M}}
\def\cN{\mathcal{N}}

\def\cS{\mathcal{S}}

\def\cX{\mathcal{X}}
\def\cY{\mathcal{Y}}
\def\cZ{\mathcal{Z}}

\usepackage{lastpage}
\jmlrheading{22}{2021}{1-\pageref{LastPage}}{11/20; Revised
7/21}{11/21}{20-1251}{Chong Liu, Yuqing Zhu, Kamalika Chaudhuri, and Yu-Xiang Wang}

\ShortHeadings{Revisiting Model-Agnostic Private Learning: Faster Rates and Active Learning}{Liu, Zhu, Chaudhuri, and Wang}
\firstpageno{1}

\begin{document}

\title{Revisiting Model-Agnostic Private Learning: Faster Rates and Active Learning}

\author{\name Chong Liu$^\dagger$ \email chongliu@cs.ucsb.edu 
       \AND
       \name Yuqing Zhu$^\dagger$ \email yuqingzhu@ucsb.edu 
       \AND
       \name Kamalika Chaudhuri$^\star$ \email kamalika@cs.ucsd.edu 
       \AND
       \name Yu-Xiang Wang$^\dagger$  \email yuxiangw@cs.ucsb.edu
       \AND
       \addr $\dagger$ Department of Computer Science, University of California, Santa Barbara,  CA 93106, USA\\
       $\star$ Department of Computer Science and Engineering, University of California, San Diego, La Jolla, CA 92093, USA
       }

\editor{Andreas Krause}

\maketitle

\begin{abstract}

The Private Aggregation of Teacher Ensembles (PATE) framework is one of the most promising recent approaches in differentially private learning. Existing theoretical analysis shows that PATE consistently learns any VC-classes in the realizable setting, but falls short in explaining its success in more general cases where the error rate of the optimal classifier is bounded away from zero. We fill in this gap by introducing the Tsybakov Noise Condition (TNC) and establish stronger and more interpretable learning bounds. These bounds provide new insights into when PATE works and improve over existing results even in the narrower realizable setting. We also investigate the compelling idea of using active learning for saving privacy budget, and empirical studies show the effectiveness of this new idea. The novel components in the proofs include a more refined analysis of the majority voting classifier --- which could be of independent interest --- and an observation that the synthetic ``student'' learning problem is nearly realizable by construction under the Tsybakov noise condition.

\end{abstract}


\section{Introduction}


Differential privacy (DP)  \citep{dwork2006calibrating}  is one of the most popular approaches towards addressing the privacy challenges in the era of artificial intelligence and big data.  While differential privacy is certainly not a solution to all privacy-related problems,
it represents a gold standard and is a key enabler in many applications  \citep{machanavajjhala2008privacy,erlingsson2014rappor,mcmahan2018learning}. 

Recently, there has been an increasing demand in training machine learning and deep learning models with DP guarantees, which has motivated a growing body of  research on this problem  \citep{kasiviswanathan2011can,chaudhuri2011differentially,bassily2014private,wang2015privacy,abadi2016deep,shokri2015privacy}. 

In a nutshell, differentially private machine learning aims at 
providing formal privacy guarantees that provably reduce the risk of identifying individual data points in the training data, while still allowing the learned model to be deployed and to provide accurate predictions.  
Many of these methods satisfying DP guarantees work well in low-dimensional regime where the model is small and the data is large.
It however remains a fundamental challenge how to avoid the \emph{explicit} dependence in the \emph{ambient dimension} of the model
and to develop practical methods in privately releasing deep learning models with a large number of parameters.


The ``knowledge transfer'' model of differentially private learning is a promising recent development  \citep{papernot2017semi,papernot2018scalable} which relaxes the problem by giving the learner access to a public unlabeled dataset. The main workhorse of this model is the Private Aggregation of Teacher Ensembles (PATE) framework:

\fbox{\parbox{0.9\columnwidth}{The \textit{PATE} Framework:
		\small
\noindent
\begin{enumerate}
\item Randomly partition the private dataset into $K$ splits.
\item Train one ``teacher'' classifier on each split.
\item Apply the $K$ ``teacher'' classifiers on public data and \emph{privately release} their majority votes as pseudo-labels.
\item Output the ``student'' classifier trained on the pseudo-labeled public data.
\end{enumerate}}}

PATE achieves DP via the sample-and-aggregate scheme  \citep{nissim2007smooth} for releasing the pseudo-labels. Since the teachers are trained on disjoint splits of the private dataset, adding or removing one data point could affect only one of the teachers, hence limiting the influence of any single data point. The noise injected in the aggregation will then be able to ``obfuscate'' the output and obtain provable privacy guarantees.

This approach is appealing in practice as it does not place any restrictions on the \emph{teacher} classifiers, thus allowing any deep learning models to be used in a \emph{model-agnostic} fashion. The competing alternative for differentially private deep learning, NoisySGD  \citep{abadi2016deep}, is \emph{not} model-agnostic, and it requires significantly more tweaking and modifications to the model to achieve a comparable performance, (e.g., on MNIST), if achievable.

There are a number of DP mechanisms that can be used to instantiate the PATE Framework. Laplace mechanism and Gaussian mechanism are used in \citet{papernot2017semi,papernot2018scalable} respectively. 
This paper primarily considers the new mechanism of \citet{bassily2018model}, which instantiates the PATE framework with a more data-adaptive scheme of private aggregation based on the Sparse Vector Technique (SVT). This approach allows PATE to privately label many examples while paying a privacy loss for only a small subset of them (see Algorithm~\ref{alg-priv-agg} for details). Moreover,  \citet{bassily2018model} provides the first theoretical analysis of PATE which shows that it is able to PAC-learn any hypothesis classes with finite VC-dimension in the realizable setting, i.e, expected risk of best hypothesis equals $0$. And in this case, the center of teacher agreement is true label. However, this is a giant leap from the standard differentially private learning models (without the access to a public unlabeled dataset) because the VC-classes are \emph{not} privately learnable in general \citep{bun2015differentially,wang2016learning}.  \citet{bassily2018model} also establishes a set of results on the agnostic learning setting, albeit less satisfying, as the \emph{excess risk}, i.e., the error rate of the learned classifier relative to the optimal classifier, does not vanish as the number of data points increases, a.k.a., inconsistency.

\begin{table*}[t]
	\centering
	\caption{Summary of our results: excess risk bounds for PATE algorithms.  }
	\label{tab:summary}
	\resizebox{6in}{!}{
		\begin{tabular}{ccccc}
			\noalign{\smallskip} \hline \noalign{\smallskip}
			\multirow{2}{*}{\textbf{Algorithm}} & \multirow{2}{*}{\begin{tabular}[c]{@{}c@{}}\textbf{PATE} (Gaussian Mech.)\\  \citet{papernot2017semi}\end{tabular}} & \multicolumn{2}{c}{\textbf{PATE} (SVT-based)} & \multirow{2}{*}{\begin{tabular}[c]{@{}c@{}}\textbf{PATE} (Active Learning)\\ This paper\end{tabular}} \\
			&                       &  \citet{bassily2018model}   & This paper   &    \\ \noalign{\smallskip} \hline \noalign{\smallskip}
			Realizable                & \blue{ $\tilde{O}\Big( \frac{d}{(n\epsilon)^{2/3}} \vee\frac{d}{m}\Big)$      }             & $\tilde{O}\Big(\frac{d}{(n\epsilon)^{2/3}} \vee \sqrt{\frac{d}{m}}\Big)$                     & \blue{$\tilde{O}\Big(\frac{d^{3/2}}{n\epsilon}\vee \frac{d}{m}\Big)$}          & \blue{$\tilde{O}\Big(\frac{d^{3/2}\theta^{1/2}}{n\epsilon}\vee\frac{d}{m}\Big)$}                                                                       \\ \noalign{\smallskip} \hline \noalign{\smallskip}
			$\tau$-TNC                       & \blue{ $\tilde{O}\Big(\big(\frac{d^{3/2}}{n\epsilon}\big)^\frac{2\tau}{4-\tau}\vee\frac{d}{m}\Big)$ }                    & same as agnostic                     & \blue{$\tilde{O}\Big(\big(\frac{d^{3/2}}{n\epsilon}\big)^\frac{\tau}{2-\tau} \vee \frac{d}{m}\Big)$}         & \blue{$\tilde{O}\Big( \big(\frac{d^{3/2}\theta^{1/2}}{n\epsilon}\big)^\frac{\tau}{2-\tau} \vee \frac{d}{m}\Big)$}                                                                             \\ \noalign{\smallskip} \hline \noalign{\smallskip}
			\multirow{3}{*}{\begin{tabular}[c]{@{}c@{}}Agnostic\\ (vs $h^*$)\end{tabular}}               &  \multirow{3}{*}{\blue{$\Omega(\mathtt{Err}(h^*))$  required.}}                  &  \multirow{3}{*}{\begin{tabular}[c]{@{}c@{}}$13\mathtt{Err}(h^*)+$\\ $\tilde{O}\Big(\frac{d^{3/5}}{n^{2/5}\epsilon^{2/5}} \vee \sqrt{\frac{d}{m}}\Big)$\end{tabular}}                  & \multirow{3}{*}{\blue{$\Omega(\mathtt{Err}(h^*))$  required.}}         & \multirow{3}{*}{\blue{$\Omega(\mathtt{Err}(h^*))$  required.}}\\
			& & & & \\
			& & & & \\\noalign{\smallskip} \hline \noalign{\smallskip}
			\multirow{2}{*}{\begin{tabular}[c]{@{}c@{}}Agnostic\\(vs $h^\mathtt{agg}_\infty$)\end{tabular}}                &  \multirow{2}{*}{-}                   &  \multirow{2}{*}{-}                   & \multirow{2}{*}{\begin{tabular}[c]{@{}c@{}}\blue{Consistent under}\\\blue{weaker conditions.}\end{tabular}}
			& \multirow{2}{*}{-}\\
			& & & & \\\noalign{\smallskip} \hline \noalign{\smallskip}
	\end{tabular}}
	\scriptsize{
		\begin{itemize}
			\itemsep-0.5em
			\item Results new to this paper are highlighted in \blue{blue}.
			\item Teacher number hyperparameter $K$ is chosen optimally. The number of public data points we privately label is chosen optimally (subsampling the available public data to run PATE) to minimize the risk bound. $\delta$ is assumed to be in its typical range $\delta< 1/\textrm{poly}(n)$ and $\epsilon< \log(1/\delta)$. The TNC parameter $\tau$ ranges between $(0,1]$. See Table \ref{tab:notations} for a checklist of notations.
			\item Proofs of utility guarantees of PATE (Gaussian mechanism) can be found in Appendix \ref{sec:refined_proofs}.
		\end{itemize}
	}
\end{table*}

To fill in the gap, in this paper, we revisit the problem of model-agnostic private learning in PATE framework in two non-realizable settings: under the Tsybakov Noise Condition (TNC) \citep{mammen1999smooth,tsybakov2004optimal} and in agnostic setting. By making TNC assumption, teachers stay close to the best hypothesis $h^*$ in hypothesis class, thus we consider $h^*$ as the new center for teachers to agree on, instead of considering true label in the realizable setting. We make no assumptions in agnostic setting, and a different center of teacher gravity is considered. In addition, we introduce active learning \citep{hanneke2014theory} to PATE and propose a new practical algorithm.

\paragraph{Summary of results.} Our contributions are summarized as follows.
\begin{enumerate}
	\item We show that PATE consistently learns any VC-classes under TNC with fast rates and requires very few unlabeled public data points.  When specializing to the realizable case, we show that the sample complexity bound of the SVT-based PATE is  $\tilde{O}(d^{3/2}/\alpha\epsilon)$  and $\tilde{O}(d/\alpha)$ for the private and public datasets respectively.  The best known results \citep{bassily2018model} is $\tilde{O}(d^{3/2}/\alpha^{3/2}\epsilon)$  (for private data) and $\tilde{O}(d/\alpha^2)$ (for public data).
	\item We analyze standard Gaussian mechanism-based PATE \citep{papernot2018scalable} under TNC. In the realizable case, we obtained a sample complexity of $\tilde{O}(d^{3/2}/\alpha\epsilon)$ and $\tilde{O}(d/\alpha)$ for the private and public datasets respectively, which matches the bound of \citep{bassily2018model} with a simpler and more practical algorithm that uses fewer public data points.
	\item We show that PATE learning is \emph{inconsistent}  for agnostic learning in general and derive new learning bounds that compete against a sequence of limiting majority voting classifiers.
	\item We propose a new active learning-based algorithm, PATE with Active Student Queries (PATE-ASQ), to adaptively select which public data points to release. Under TNC, we show that active learning with standard Gaussian mechanism is able to match the same learning bounds of the SVT-based method for privacy aggregation (Algorithm~\ref{alg-priv-agg-pate}), except some additional dependence.
	\item Finally, our experiments on real-life datasets demonstrate that PATE-ASQ achieves significantly better accuracy than standard PATE algorithms while incurring the same or lower privacy loss.
\end{enumerate}

These results  (summarized in Table \ref{tab:summary}) provide strong theoretical insight into how PATE works.   Interestingly, our theory suggests that \emph{Gaussian mechanism suffice}s especially if we use active learning and that it is better \emph{not} to label all public data when the number of public data points $m$ is large.  The remaining data points can be used for semi-supervised learning.  These tricks have been
proposed in \emph{empirical} studies of PATE (see, e.g., semi-supervised learning \citep{papernot2017semi,papernot2018scalable}, active learning \citep{zhao2019improving}), thus our \emph{theory} can be viewed as providing formal justifications to these PATE variants that are producing strong empirical results in \emph{deep learning with differential privacy}.

\paragraph{Motivation and applicability.} We conclude the introduction by commenting on the applicability of the knowledge transfer model of differentially private learning and PATE. First, while this model applies only to those cases when a (small) public unlabeled dataset is available, it gains a more favorable privacy-utility tradeoff on those applicable cases. Second, public datasets are often readily available (e.g., census microdata) or can be acquired at a low cost (e.g., incentivizing patients to opt-in) especially if we do not need labels (e.g., getting doctor's diagnosis is expensive). Note that this setting is different from label differential privacy \citep{chaudhuri2011sample} where only labels are considered private. In our problem, even if the public data points are labeled, they are scarce and learning directly from them without using the private data will not give the same learning bound. In addition, PATE uses standard off-the-shelf learners / optimizers as blackboxes, thereby retaining their computational efficiency. For these reasons, we argue that the ``knowledge transfer'' model is widely applicable and could enable practical algorithms with formal DP guarantees  in the many applications where the standard private learning model fails to be sufficiently efficient, private and accurate at the same time.

\section{Related Work}\label{sec-rw}

The literature on differentially private machine learning is enormous and it is impossible for us to provide an exhaustive discussion. Instead we focus on a few closely related work and only briefly discuss other representative results in the broader theory of private learning. 

\subsection{Private Learning with an Auxiliary Public Dataset}
The use of an auxiliary unlabeled public dataset was pioneered in empirical studies \citep{papernot2017semi,papernot2018scalable} where PATE was proposed and shown to produce stronger results than NoisySGD in many regimes. 
Our work builds upon \citet{bassily2018model}'s first analysis of PATE  and substantially improves the theoretical underpinning. To the best of our knowledge, our results are new and we are the first that consider \emph{noise models} and \emph{active learning} for PATE. 

\citet{bassily2019limits} also studied the problem of private learning with access to an additional public dataset. 
 Specifically, their result reveals an interesting ``theorem of the alternatives''-type result that says either a VC-class is learnable without an auxiliary public dataset, or we need at least $m = \Omega(d/\alpha)$ public data points, which essentially says that our sample complexity on the (unlabeled) public data points are optimal. They also provide an upper bound that says $\tilde{O}(d/\alpha^2)$ private data and $\tilde{O}(d/\alpha)$ public data are sufficient (assuming constant privacy parameter $\epsilon$) to \emph{agnostically learn} any classes with VC-dimension $d$ to $\alpha$-excess risk. Their algorithm however uses an explicit (distribution-independent)  $\alpha$-net  construction due to \citet{beimel2016private} and exponential mechanism for producing pseudo-labels, which cannot be efficiently implemented.  Our contributions are complementary as we focus on \emph{oracle-efficient} algorithms that reduce to the learning bounds of ERM oracles (for passive learning) and active learning oracles. Our algorithms can therefore be implemented (and has been) in practice \citep{papernot2017semi,papernot2018scalable}. Moreover, we show that under TNC, the inefficient construction is not needed and PATE is indeed consistent and enjoys faster rates.   It remains an open problem how to achieve consistent private agnostic learning with only access to ERM oracles.
 
 
 
\subsection{Privacy-Preserving Prediction}
  There is another line of work \citep{dwork2018privacy} that focuses on the related problem of ``privacy-preserving prediction'' which does not release the learned model (which we do), but instead privately answer one randomly drawn query $x$ (which we need to answer many, so as to train a model that can be released).  While their technique can be used to obtain bounds in our setting, it often involves weaker parameters. 
  More recent works under this model \citep[see e.g.,][]{dagan2020pac,nandi2020privately} notably achieve consistent agnostic learning in this setting with rates comparable to that of \citet{bassily2019limits}. However, they rely on the same explicit $\alpha$-net construction \citep{beimel2016private}, which renders their algorithm computationally inefficient in practice.  In contrast, we analyze an oracle-efficient algorithm via a reduction to supervised learning  (which is practically efficient if we believe supervised learning is easy).

\subsection{Theory of Private Learning}
More broadly, the learnability and sample complexity of private learning were studied under various models in \citet{kasiviswanathan2011can,beimel2013characterizing,beimel2016private,chaudhuri2011sample,bun2015differentially,wang2016learning,bassily2019limits}. The VC-classes were shown to be learnable when the either the hypothesis class or the data-domain is finite \citep{kasiviswanathan2011can}. \citet{beimel2013characterizing} characterizes the sample complexity of private learning in the realizable setting with a particular ``dimension''  that measures the extent to which we can construct a specific discretization of the hypothesis space that works for ``all distributions'' on data.  Such a discretization does not exist,  
when $\cH$ and $\cX$ are both continuous. Specifically, the problem of learning threshold functions on $[0,1]$ having VC-dimension of $1$ is not privately learnable \citep{chaudhuri2011sample,bun2015differentially}.

\subsection{Weaker Private Learning Models}
This setting of private learning was relaxed in various ways to circumvent the above artifact. These include protecting only the labels \citep{chaudhuri2011sample,beimel2016private}, leveraging prior knowledge with a prior distribution \citep{chaudhuri2011sample}, switching to the general learning setting with Lipschitz losses \citep{wang2016learning}, relaxing the distribution-free assumption \citep{wang2016learning}, and the setting we consider in this paper --- when we assume the availability of an auxiliary public data \citep{bassily2018model,bassily2019limits}. Note that these settings are closely related to each other in that some additional information about the distribution of the data is needed.

\subsection{Tsybakov Noise Condition and Statistical Learning Theory}
The Tsybakov Noise Condition (TNC) \citep{mammen1999smooth,tsybakov2004optimal} is a natural and well-established condition in learning theory that has long been used in the analysis of passive as well as active learning \citep{boucheron2005theory}. The Tsybakov noise condition is known to yield better convergence rates for passive learning \citep{hanneke2014theory}, and label savings for active learning \citep{zhang2014beyond}. However, the contexts under which we use these techniques are different. For instance, while we are making the assumption of TNC, the purpose is not for active learning, but rather to establish stability. When we apply active learning, it is for the synthetic learning problem with pseudo-labels that we release privately, which does not actually satisfy TNC. To the best of our knowledge, we are the first that formally study noise models in the theory of private learning. 
Lastly, active learning was considered for PATE learning in \citep{zhao2019improving}, which demonstrates the clear practical benefits of adaptively selecting what to label. We remain the first that provides theoretical analysis with provable learning bounds.


\begin{table}[p]
\centering
\caption{Summary of symbols and notations.}\label{tab:notations}
\resizebox{\textwidth}{!}{
\begin{tabular}{ccl}
\noalign{\smallskip} \hline
\textbf{Symbol}  & \textbf{Definition} & \textbf{Description} \\ \hline
$\1(x)$     &  $=1 (x=\mathtt{T}), =0 (x=\mathtt{F})$     &  indicator function     \\ \hline 
$\mathtt{Err}(h)$             &  $\E_{(x,y)\sim \cD} [\1(h(x) \neq y)]$ & expected risk of $h$ w.r.t. $\cD$     \\  \hline 
$\widehat{\mathtt{Err}}(h)$ &$\frac{1}{n}\sum_{i=1}^n [\1(h(x_i) \neq y_i)]$ &empirical risk of  $h$ w.r.t. dataset $\{(x_i,y_i)|i \in [n]\}$     \\ \hline 
$\cD$    &         &  distribution over $\cZ$     \\  \hline
$d$       &      &  VC dimension     \\ \hline
$\cD_\cX$       &      & marginal distribution over $\cX$     \\\hline 
$D^T$             &  $\{(x^T_i,y^T_i)|i \in [n]\} \sim \cD$ &labeled private teacher dataset     \\ \hline 
$D^S$             &  $\{(x^S_j)|j \in [m]\} \sim \cD_\cX$ &unlabeled public student dataset     \\ \hline 
$\mathtt{DIS}$      &       & region of disagreement in active learning    \\ \hline 
$\Dis (h_1, h_2)$       & $\E_{x \sim \cD_\cX} [\1 (h_1(x) \neq h_2(x)]$     & expected disagreement of $h_1$ and $h_2$ w.r.t $\cD$     \\ \hline 
$\widehat{\Dis} (h_1, h_2)$ & $\frac{1}{n}\sum_{i=1}^n [\1 (h_1(x_i) \neq h_2(x_i))]$       &     empirical disagreement of $h_1$ and $h_2$ w.r.t. $\{(x_i,y_i)|i \in [n]\}$     \\ \hline 
$\cH$           &  $\cH \subseteq \{0,1\}^\cX$ &hypothesis class     \\ \hline
$h$            &  &hypothesis, a function mapping from $\cX$ to $\cY$     \\ \hline 
$h^*$            &  $\argmin_{h \in \cH} \mathtt{Err}(h)$ & best hypothesis     \\ \hline 
$\hat{h}$            &  $\argmin_{h \in \cH} \widehat{\mathtt{Err}}(h)$ &Empirical Risk Minimizer (ERM)  \\ \hline
$\hat{h}^\mathtt{agg}$    &        &  aggregated classifier in PATE  \\ \hline
$\hat{h}^\mathtt{priv}$    &        &  privately aggregated classifier in PATE  \\ \hline
$h^\mathtt{agg}_\infty$    &        &  infinite ensemble classifier  \\ \hline
$K$ &           &  number of teachers     \\ \hline 
$\ell$ &           &  labeling budget     \\ \hline 
$m$ &           &  number of unlabeled student points     \\ \hline 
$n$  &          &  number of labeled teacher points     \\ \hline 
$[n]$             &  $\{1,2,...,n\}$    & integer set \\ \hline
$\tilde{O}$ & & big O notation hiding poly-logarithmic factors \\ \hline
$r(x)$ & $\E[y|x]$ & regression function from $x$ to $y$ \\ \hline
$T$ &           & cut-off threshold     \\ \hline 
$\cX$   &         &  feature space     \\  \hline 
$\cY$            &  $\{0,1\}$ & label space     \\  \hline 
$\cZ$            &  $\cX \times \cY$ &sample space     \\  \hline
$\cZ^*$            &  $\bigcup_{n\in\mathbb{N}}\cZ^n$ &space of a dataset of unspecified size    \\  \hline 
$\alpha$ &           &  excess risk     \\ \hline 
$\beta, \gamma$    &        &  failure probabilities     \\ \hline
$\epsilon, \delta$  & Definition \ref{def-dp}         &  parameters of differential privacy     \\ \hline
$\nu, \xi$ &Definition \ref{def-app_high_margin} & parameters of high margin condition \\ \hline
$\tau$ & Definition \ref{def-original-tnc}          &  parameter of the Tsybakov noise condition     \\ \hline
$\theta$ & \cite{hanneke2014theory}          &  disagreement coefficient of active learning    \\ \hline 
$\widehat{\Delta}$ &Eq. \ref{eq:margin} & realized margin \\ \hline
$\Delta$ &Eq. \ref{eq:exp_margin} & expected margin \\ \hline
$\perp$ & & randomly assigned label \\ \hline
$\vee$ & $ X \vee Y=\max\{X,Y\}$ & $\max$ operation \\ \hline
$\lesssim, \gtrsim$ & & inequalities hiding logarithmic factors\\ \hline
$c, c', C$ & & constants\\ \hline
\end{tabular}}
\end{table}

\section{Preliminaries}
In this section, first we introduce symbols and notations that we will use throughout this paper. Then we formally introduce differential privacy and discuss existing progress on PATE and model-agnostic private learning. Finally we introduce disagreement-based active learning, which is the key tool we will use for our new active learning-based PATE algorithm.

\subsection{Symbols and Notations}
We use $[n]$ to denote the set $\{1,2,...,n\}$. Let $\cX$ denote the feature space, $\cY = \{0,1\}$ denote the label, $\cZ = \cX \times \cY$ to denote the sample space, and $\cZ^*=\bigcup_{n\in\mathbb{N}}\cZ^n$ to denote the space of a dataset of unspecified size. A hypothesis (classifier) $h$ is a function mapping from $\cX$ to $\cY$. A set of hypotheses $\cH \subseteq \{0,1\}^\cX$ is called the hypothesis class. The VC dimension of $\cH$ is denoted by $d$. Also, let $\cD$ denote the distribution over $\cZ$, and $\cD_\cX$ denote the marginal distribution over $\cX$. $D^T=\{(x^T_i,y^T_i)|i \in [n]\} \sim \cD$ is the labeled private teacher dataset, and $D^S=\{(x^S_j)|j \in [m]\} \sim \cD_\cX$ is the unlabeled public student dataset.

The expected risk of a certain hypothesis $h$ with respect to the distribution $\cD$ over $\cZ$ is defined as $\mathtt{Err}(h) = \E_{(x,y)\sim \cD} [\1(h(x) \neq y)]$, where $\1(x)$ is the indicator function which equals to $1$ when $x$ is true, $0$ otherwise. The empirical risk of a certain hypothesis $h$ with respect to a dataset $\{(x_i,y_i)|i \in [n]\}$ is defined as $\widehat{\mathtt{Err}}(h) = \frac{1}{n}\sum_{i=1}^n [\1(h(x_i) \neq y_i)]$. The best hypothesis $h^*$ is defined as $h^* = \argmin_{h \in \cH} \mathtt{Err} (h)$, and the Empirical Risk Minimizer (ERM) $\hat{h}$ is defined as $\hat{h} = \argmin_{h \in \cH} \widehat{\mathtt{Err}} (h)$. $\hat{h}^\mathtt{agg}$ is used to denote the aggregated classifier in the PATE framework. $\hat{h}^\mathtt{priv}$ denotes the privately aggregated one. The expected disagreement between a pair of hypotheses $h_1$ and $h_2$ with respect to the distribution $\cD_\cX$ is defined as $\Dis (h_1, h_2) = \E_{x \sim \cD_\cX} [\1 (h_1(x) \neq h_2(x)]$. The empirical disagreement between a pair of hypotheses $h_1$ and $h_2$ with respect to a dataset $\{(x_i,y_i)|i \in [n]\}$ is defined as $\widehat{\Dis} (h_1, h_2) = \frac{1}{n}\sum_{i=1}^n [\1 (h_1(x_i) \neq h_2(x_i))]$. Throughout this paper, we use standard big $O$ notations; and to improve the readability, we use $\lesssim$ and $\tilde{O}$ to hide poly-logarithmic factors. For reader's easy reference, we summarize the symbol and notations above in Table~\ref{tab:notations}.

\subsection{Differential Privacy and Private Learning}
Now we formally introduce differential privacy.
\begin{definition}[Differential Privacy  \citep{dwork2014algorithmic}]\label{def-dp}
	A randomized algorithm $\cM: \cZ^* \rightarrow \mathcal{R}$ 
	is ($\epsilon, \delta$)-DP (differentially private) if for every pair of \emph{neighboring} datasets $D,D'\in \cZ^*$ (denoted by $\|D - D'\|_1 = 1$) for all $\cS \subseteq \mathcal{R}$:
	\begin{align*}
	\P(\cM(D) \in \cS) \leq e^\epsilon \cdot \P(\cM(D') \in \cS) + \delta.
	\end{align*}
\end{definition}
The definition says that if an algorithm  $\cM$ is DP, then no adversary can use the output of $\cM$ to distinguish between two parallel worlds where an individual is in the dataset or not.
$\epsilon,\delta$ are privacy loss parameters that quantify the strength of the DP guarantee. The closer they are to $0$, the stronger the guarantee is.


The problem of DP learning aims at designing a randomized training algorithm that satisfies Definition~\ref{def-dp}. More often than not, the research question is about understanding the privacy-utility trade-offs and characterizing the Pareto optimal frontiers.

\subsection{PATE and Model-Agnostic Private Learning}
There are different ways we can instantiate the PATE framework to privately aggregate the teachers' predicted labels. The simplest, described in Algorithm~\ref{alg-priv-agg-pate}, uses Gaussian mechanism to perturb the voting score.

\begin{algorithm}[t]
	\caption{Standard PATE  \citep{papernot2018scalable}} 
	\label{alg-priv-agg-pate}
	{\bf Input:}
	``Teachers'' $\hat{h}_1,...,\hat{h}_K$ trained on \emph{disjoint} subsets of the private data.  
	``Nature'' chooses an \emph{adaptive} sequence of data points $x_1,...,x_\ell$.  Privacy parameters $\epsilon, \delta > 0$.
	\begin{algorithmic}[1]
		\STATE Find $\sigma$ such that $\sqrt{\frac{2\ell\log(1/\delta)}{\sigma^2}} + \frac{\ell}{2\sigma^2} = \epsilon$.
		\STATE Nature chooses $x_1$.
		\FOR{$j \in [\ell]$}
		\STATE Output $\hat{y}_j \leftarrow \1(\sum_{k=1}^K\hat{h}_k(x_j) + \cN(0,\sigma^2) \geq K/2)$. 
		\STATE	Nature chooses $x_{j+1}$ adaptively (as a function of the output vector till time $j$).
		\ENDFOR
	\end{algorithmic}
\end{algorithm}

An alternative approach due to \citep{bassily2018model} uses the Sparse Vector Technique (SVT) in a nontrivial way to privately label substantially more data points in the cases when teacher ensemble's predictions are \emph{stable} for most input data. The stability is quantified in terms of the margin function,  defined as 
\begin{equation}\label{eq:margin}
\widehat{\Delta}(x) :=  \Big|2 \sum_{k=1}^K\hat{h}_k(x)-K\Big|,
\end{equation}
which measures the absolute value of the difference between the number of votes (see Algorithm~\ref{alg-priv-agg}).
\begin{algorithm}[t]
	\caption{SVT-based PATE  \citep{bassily2018model}}
	\label{alg-priv-agg}
	{\bf Input:}
	``Teacher'' classifiers $\hat{h}_1,...,\hat{h}_K$ trained on \emph{disjoint} subsets of the private data. 
	``Nature'' chooses an \emph{adaptive} sequence of data points $x_1,...,x_\ell$.
	Unstable cut-off $T$, privacy parameters $\epsilon, \delta > 0$.
	
	\begin{algorithmic}[1]
		\STATE Nature chooses $x_1$.
		\STATE $\lambda \leftarrow  (\sqrt{2T(\epsilon + \log(2/\delta))} + \sqrt{2T\log(2/\delta)})/\epsilon$.
		\STATE $w \leftarrow 3 \lambda \log (2(\ell+T) /\delta), \hat{w} \leftarrow w + \mathtt{Lap}(\lambda)$.
		\STATE $c = 0$.
		\FOR{$j \in [\ell]$}
		\STATE $\mathrm{dist}_j  \leftarrow  \max\{0, \lceil\widehat{\Delta}(x_j)/2\rceil-1\}$.
		\STATE $\widehat{\mathrm{dist}}_{j} \leftarrow \mathrm{dist}_j + \mathtt{Lap}(2\lambda)$.
		\IF {$\widehat{\mathrm{dist}}_j > \hat{w}$}
		\STATE Output $\hat{y}_j \leftarrow \1(\sum_{k=1}^K\hat{h}_k(x_j) \geq K/2) $.
		\ELSE
		\STATE Output $\hat{y}_j \leftarrow \perp$.
		\STATE $c \leftarrow c + 1$, break if $c\geq T$.
		\STATE $\hat{w} \leftarrow w + \mathtt{Lap}(\lambda)$.
		\ENDIF
		\STATE	Nature chooses $x_{j+1}$ adaptively (based on $\hat{y}_1,...,\hat{y}_{j}$).
		\ENDFOR
	\end{algorithmic}
\end{algorithm}




In both algorithms, the privacy budget parameters $\epsilon,\delta$ are taken as an input and the following privacy guarantee applies to all input datasets.
\begin{theorem}\label{thm-pg-agg}
	Algorithm~\ref{alg-priv-agg-pate}~and~\ref{alg-priv-agg} are both ($\epsilon,\delta$)-DP.
\end{theorem}
Careful readers may note the slightly improved constants in the formula for calibrating privacy than when these methods were first introduced. We include the new proof based on the \emph{concentrated differential privacy}  \citep{bun2016concentrated} approach in the Appendix \ref{sec:refined_proofs}.

	The key difference between the two private-aggregation mechanisms is that the standard PATE pays for a unit privacy loss for every public data point labeled, while the SVT-based PATE essentially pays only for those queries where the voted answer from the teacher ensemble is close to be unstable (those with a small margin).  
	Combining this intuition with the fact that the individual classifiers are accurate ---  by the statistical learning theory, they are --- the corresponding majority voting classifier can be shown to be accurate with a large margin. These two critical observations of \citet{bassily2018model} lead to the first learning theoretic guarantees for SVT-based PATE. For completeness, we include this result with a concise new proof  in Appendix~\ref{sec:refined_proofs}.  

\begin{lemma}[Adapted from Theorem~3.11 of  \citet{bassily2018arxiv}]\label{lem:utility_svt}
	If the classifiers $\hat{h}_1,...,\hat{h}_K$ and the sequence $x_1,...,x_\ell$ obey that there are at most $T$ of them such that 
	$
	\widehat{\Delta}(x_k) < K/3
	$
	for $K =  136\log(4\ell T/ \min(\delta,\beta/2))\cdot \sqrt{T\log(2/\delta)}/\epsilon$. Then with probability at least $1-\beta$, Algorithm~\ref{alg-priv-agg} finishes all $\ell$ queries and for all $i\in[\ell]$ such that $\widehat{\Delta}(x_i)\geq K/3$, the output of Algorithm~\ref{alg-priv-agg} is $\hat{h}^{\mathtt{agg}}(x_i)$.
\end{lemma}

\begin{lemma}[Lemma 4.2 of  \citet{bassily2018arxiv}]\label{lem:pigeon_hole}
	If the classifiers $\hat{h}_1,....,\hat{h}_K$ obey that each of them makes at most $B$ mistakes on data $(x_1,y_1),...,(x_\ell,y_\ell)$, then 
	$$
	\bigg|	\Big\{ i\in[\ell]\; \Big|\;  \sum_{k=1}^K \1(\hat{h}_k(x_i) \neq y_i) \geq \frac{K}{3} \Big\}\bigg| \leq  3B.
	$$  
\end{lemma}
Lemma~\ref{lem:pigeon_hole} implies that if the individual classifiers are accurate ---  by the statistical learning theory, they are --- the corresponding majority voting classifier is not only nearly as accurate, but also has sufficiently large margin that satisfies the conditions in Lemma~\ref{lem:utility_svt}.

Next, we state and provide a straightforward proof of the following results due to \citep{bassily2018arxiv}. The results are already stated in the referenced work in the form of sample complexities, but we include a more direct analysis of the error bound and clarify a few  technical subtleties. 

\begin{algorithm}[!htbp]
	\caption{PATE-PSQ}\label{alg-psq}
	{\bf Input:}
	Labeled private teacher dataset $D^T$, unlabeled public student dataset $D^S$, unstable query cutoff $T$, privacy parameters $\epsilon, \delta > 0$; number of splits $K$.
	\begin{algorithmic}[1]
		\STATE Randomly and evenly split the teacher dataset $D^T$ into $K$ parts $D^T_k \subseteq D^T$ where $k \in [K]$.
		\STATE Train $K$ classifiers $\hat{h}_{k}\in \cH$, one from each part $D^T_k$.
		\STATE Call Algorithm~\ref{alg-priv-agg} with parameters $(\hat{h}_1,...,\hat{h}_K), D^S,  T, \epsilon,\delta$ and $\ell=m$ to obtain pseudo-labels for the public dataset $\hat{y}^S_1,...,\hat{y}^S_m$.
		(Alternatively, call Algorithm~\ref{alg-priv-agg-pate}  with parameters $(\hat{h}_1,...,\hat{h}_K), D^S,\epsilon,\delta,\ell=m$)
		\STATE For those pseudo labels that are $\perp$, assign them arbitrarily to $\{0,1\}$.
	\end{algorithmic}
	{\bf Output:}
	$\hat{h}^S$ trained on pseudo-labeled student dataset.
\end{algorithm}

\begin{theorem}[Adapted from Theorems 4.6 and 4.7 of  \citep{bassily2018arxiv}]\label{thm-ug-bassily}
	Set
	$$T = 3\Big(\E[\mathtt{Err}(\hat{h}_1)] m + \sqrt{\frac{m\log(m/\beta)}{2}}\Big),$$
	$$K = O \Big(\frac{\log(mT / \min(\delta,\beta))\sqrt{T\log(1/\delta)}}{\epsilon}\Big).$$ Let $\hat{h}^S$ be the output of Algorithm~\ref{alg-psq} that uses Algorithm~\ref{alg-priv-agg} for privacy aggregation. With probability at least $1-\beta$ (over the randomness of the algorithm and the randomness of all data points drawn i.i.d.), we have
	$$\mathtt{Err}(\hat{h}^S) \leq  \tilde{O}\Big(\frac{d^2m \log(1/\delta)}{n^2\epsilon^2}  + \sqrt{\frac{d}{m}}\Big)$$
	for the realizable case, and
	$$\mathtt{Err}(\hat{h}^S) \leq  13\mathtt{Err}(h^*) +  \tilde{O}\Big(\frac{m^{1/3}d^{2/3}}{n^{2/3}\epsilon^{2/3}}  + \sqrt{\frac{d}{m}}\Big)$$
	for the agnostic case \footnote{The numerical constant $13$ might be improvable (and it is indeed worse than the result stated in \citet{bassily2018model}), though we decide to present this for the simplicity of the proof.}. 
\end{theorem}
We provide a self-contained proof of this result in Appendix \ref{sec:refined_proofs}.

\begin{remark}[Error bounds when $m$ is sufficiently large]\label{rmk:mislarge}
	Notice that we do not have to label all public data, so when we have a large number of public data, we can afford to choose $m$ to be smaller so as to minimize the bound. That gives us a $\tilde{O}(\frac{d}{n^{2/3}\epsilon^{2/3}})$ error bound   for the realizable case and a $ O(\mathtt{Err}(h^*)) + \tilde{O}(\frac{d^{3/5}}{n^{2/5}\epsilon^{2/5}})$ error bound for the agnostic case \footnote{These correspond to the $\tilde{O}((d/\alpha)^{3/2})$ sample complexity bound  in Theorem 4.6 of  \citep{bassily2018arxiv} for realizable PAC learning for error $\alpha$;  and the $\tilde{O}(d^{3/2}/\alpha^{5/2})$ sample complexity bound in Theorem 4.7 of  \citep{bassily2018arxiv} for agnostic PAC learning with error $O(\alpha + \mathtt{Err}(h^*) )$. The privacy parameter $\epsilon$ is taken as a constant in these results.}.
\end{remark}





\begin{algorithm}[!htbp]
	\caption{Disagreement-Based Active Learning  \citep{hanneke2014theory}}
	\label{alg-active_learning}
	{\bf Input:}
	A ``data stream'' $x_1,x_2,...$ sampled i.i.d. from distribution $\cD$. A hypothesis class $\cH$. An on-demand ``labeling service'' that outputs label $y_i \sim P(y | x = x_i)$ when requested at time $i$.
	Parameter $\ell,m,\gamma$.

	\begin{algorithmic}[1]
	\STATE Initialize the version space $V \leftarrow \cH$.
	\STATE Initialize the selected dataset $Q \leftarrow  \varnothing$.
	\STATE Initialize ``Current Output'' to be any $h\in \cH$.
	\STATE Initialize ``Counter'' $c \leftarrow 0$.
		    \FOR{$j \in [m] $}
	    \IF{$x_j \in \mathtt{DIS}(V)$}
	    
	    \STATE ``Request for label'' for $x_j$ and get back $y_j$ from the ``labeling service''.

	    \STATE Update $Q \leftarrow Q \cup \{(x_j, y_j)\}$.
	    \STATE $c \leftarrow c + 1$.
	    
	    \ENDIF

	    \IF{$\log_2(j) \in \mathbb{N}$}
	    \STATE Update $V \leftarrow \{ h\in V: (\mathtt{Err}_Q(h) - \min_{g \in V} \mathtt{Err}_Q (g) )|Q| \leq U(j, \gamma_j) j \}$,\\
	    where\\
	    $U(j,\gamma_j)= c'(d \log (\theta(d/j)) + \log (1/\gamma_j))/j + c'\sqrt{\mathtt{Err}(h^*)(d \log(\theta(\mathtt{Err}(h^*)) + \log (1/\gamma_j))/j},$
	    $c'$ is a constant,  and $\gamma_j = \gamma / (\log_2 (2j))^2$.
	    \STATE Set ``Current Output'' to be any $h\in V$.
	    \ENDIF
	    \IF{$c \geq \ell$}
	    \STATE Break.
	    \ENDIF
	    \ENDFOR
	    
	\end{algorithmic}
	{\bf Output:} Return ``Current Output''.
\end{algorithm}

\subsection{Disagreement-Based Active Learning} 
We adopt the disagreement-based active learning algorithm that comes with strong learning bounds (see, e.g., an excellent treatment of the subject in \citep{hanneke2014theory}). 
The exact algorithm, described in Algorithm~\ref{alg-active_learning}, keeps updating a subset of the hypothesis class $\cH$ called a \emph{version space} by collecting labels only from those data points from  a certain \emph{region of disagreement} and eliminates candidate hypothesis that are certifiably suboptimal. 

\begin{definition}[Region of disagreement \citep{hanneke2014theory}]
For a given hypothesis class $\cH$, its region of disagreement is defined as a set of data points over which there exists two hypotheses disagreeing with each other,
\begin{align*}
\mathtt{DIS}(\cH) = \{x \in \cX: \exists h,g \in \cH \ \mathrm{s.t.}\  h(x) \neq g(x)\}.
\end{align*}
\end{definition}

Region of disagreement is the key concept of the disagreement-based active learning algorithm. It captures the uncertainty region of data points for the current version space. The algorithm is fed a sequence of data points and runs in the online fashion, whenever there exists a data point in this region, its label will be queried. Then any \emph{bad} hypotheses will be removed from the version space.

The algorithm, as it is written is not directly implementable, as it represents the version spaces explicitly, but there are practical implementations that avoids explicitly representing the versions spaces by a reduction to supervised learning oracles. In our experiments, we implement the PATE-ASQ algorithm and show it works well in practice while no explicit region of disagreement is maintained.

\section{Main Results}
In Section~\ref{sec-main_results_tnc}~and~\ref{sec-main_results_agnostic}, we present a more refined theoretical analysis of PATE-PSQ (Algorithm~\ref{alg-psq}) that uses SVT-based PATE (Algorithm~\ref{alg-priv-agg}) as the subroutine. Our results provide stronger learning bounds and new theoretical insights under various settings. In Section~\ref{sec-main_results_active}, we propose a new active learning based method and show that we can obtain qualitatively the same theoretical gain while using the simpler (an often more practical) Gaussian mechanism-based PATE (Algorithm~\ref{alg-priv-agg-pate}) as the subroutine. For comparison, we also include an analysis of standard PATE (with Gaussian mechanism) in Appendix \ref{sec:refined_proofs}. Table \ref{tab:summary} summarizes these technical results.



\subsection{Improved Learning Bounds under TNC}\label{sec-main_results_tnc}
Recall that our motivation is to analyze PATE in the cases when the best classifier does not achieve 0 error and that existing bound presented in Theorem \ref{thm-ug-bassily} is vacuous if $\mathtt{Err}(h^*)> 1/26$. The error bound of $\hat{h}^S$ does not match the performance of $h^*$ even as $m,n\rightarrow \infty$ and even if we output the voted labels without adding noise. This does not explain the empirical performance of Algorithm~\ref{alg-psq} reported in \citet{papernot2017semi,papernot2018scalable} which demonstrates that the retrained classifier from PATE could get quite close to the best non-private baselines even if the latter are far from being perfect.  For instance, on Adult dataset and SVHN dataset, the non-private baselines have accuracy 85\% and 92.8\% and PATE achieves 83.7\% and 91.6\%  respectively. 

To under stand how PATE works in the regime where the best classifier $h^*$ obeys that $\mathtt{Err}(h^*) > 0$, we introduce a large family of learning problems that satisfy the so-called Tsybakov Noise Condition (TNC), under which we show that PATE is consistent with fast rates. To understand TNC, we need to introduce a few more notations. Let label $y \in \{0,1\}$ and define the regression function $r(x) = \E[y|x]$. The Tsybakov noise condition is defined in terms of the distribution of $r(x)$. 
\begin{definition}[Tsybakov noise condition]\label{def-original-tnc}
	The joint distribution of the data $(x,y)$ satisfies the Tsybakov noise condition with parameter $\tau$ if there exists a universal constant $C>0$ such that for all $t \geq 0$
	$$
    \P( |r(x)| \leq t  ) \leq C t^\frac{\tau}{1-\tau}.
	$$
\end{definition}
Note that when $r(x) = 0.5$, the label is purely random and when $r(x) = 0$ or $1$, $y$ is a deterministic function of $x$. The Tsybakov noise condition essentially is reasonable “low noise” condition that does not require a uniform lower bound of $|r(x)|$ for all $x$. When the label-noise is bounded for all $x$, e.g., when $y = h^*(x)$ with probability $0.6$ and $1-h^*(x)$ with probability $0.4$, then the Tsybakov noise condition holds with $\tau = 1$. The case when $\tau=1$ is also known as the \emph{Massart noise condition} or \emph{bounded noise condition} in the statistical learning literature.

For our purpose, it is more convenient to work with the following equivalent definition of TNC, which is equivalent to Definition~\ref{def-original-tnc} (see a proof from \citet[Definition~7]{bousquet2004introduction}).
\begin{lemma}[Equivalent definition of TNC]\label{def-tsy}
	We say that a distribution of $(x,y)$ satisfies the Tsybakov noise condition with parameter $\tau\in [0,1]$ \emph{if and only if} there exists $\eta \in [1, \infty)$ such that, for every labeling function $h$,
	\begin{equation}\label{eq-tnc}
	\Dis (h, h_{\mathrm{Bayes}}) \leq \eta (\mathtt{Err}(h) - \mathtt{Err}(h_{\mathrm{Bayes}}))^\tau.
	\end{equation}
	where $h_{\mathrm{Bayes}}(x) = \mathbbm{1}(r(x) > 0.5)$ is the Bayes optimal classifier.
\end{lemma}
In the remainder of this section, we make the assumption that the Bayes optimal classifier $h_{\mathrm{Bayes}}\in \cH$ and works with the slightly weaker condition that requires \eqref{eq-tnc} to hold only for $h\in \cH$  and that we replace $h_{\mathrm{Bayes}}$ by the optimal classifier $h^*\in \cH$ \footnote{This slightly different condition, that requires \eqref{eq-tnc} to hold only for $h\in \cH$ but with $h_{\mathrm{Bayes}}$ replaced by the optimal classifier $h^*$ (without assuming that $h^* = h_{\mathrm{Bayes}}$) is all we need. This is formally referred to as the Bernstein class condition by  \citet{hanneke2014theory}. Very confusingly, when the Tsybakov noise condition is being referred to in more recent literature, it is in fact the Bernstein class condition --- a slightly weaker but more opaque definition about both the hypothesis class $\cH$ and the data generating distribution.}.

We emphasize that the Tsybakov noise condition is not our invention. It has a long history from statistical learning theory to interpolate between the realizable setting and the agnostic setting. Specifically, problems satisfying TNC admit fast rates. For $\tau \in [0,1]$, the empirical risk minimizer achieves an excess risk of $O(1/n^{1/(2-\tau)})$, which clearly interpolates the realizable case of $O(1/n)$ and the agnostic case of $O(1/\sqrt{n})$.



Next, we give a novel analysis of Algorithm~\ref{alg-psq} under TNC. The analysis is simple but revealing, as it not only avoids the strong assumption that requires $\mathtt{Err}(h^*)$ to be close to $0$, but also achieves a family of fast rates which significantly improves the sample complexity of PATE learning even for the realizable setting.

\begin{theorem}[Utility guarantee of Algorithm \ref{alg-psq} under TNC]\label{thm-ug-psq}
	Assume the data distribution $\cD$ and the hypothesis class $\cH$ obey the Tsybakov noise condition with parameter $\tau$. Then Algorithm~\ref{alg-psq} with
$$T = \tilde{O}\bigg( \Big(\frac{ m^{2-\tau} d^{\tau}}{n^{\tau} \epsilon^{\tau}}\Big)^\frac{2}{4-3\tau} \bigg),$$
	$$K = O \Big(\frac{\log(mT / \min(\delta,\beta))\sqrt{T\log(1/\delta)}}{\epsilon}\Big),$$
	obeys that with probability at least $1-\beta$: 
	\begin{align*}
	\mathtt{Err}(\hat{h}^S) \leq \mathtt{Err}(h^*)  + \tilde{O}\bigg(\frac{d}{m} +  \left(\frac{m d^{2}}{n^{2} \epsilon^{2}}\right)^{\frac{\tau}{4-3\tau}}    \bigg).
	\end{align*}
\end{theorem}

\begin{remark}[Bounded noise case]
	When $\tau = 1$, the Tsybakov noise condition is implied by the bounded noise assumption, a.k.a., Massart noise condition, where the labels are generated by the Bayes optimal classifier $h^*$ and then toggled with a fixed probability less than $0.5$. Theorem~\ref{thm-ug-psq} implies that the excess risk is bounded by $\tilde{O}(\frac{d^2m}{n^2\epsilon^2} + \frac{d}{m})$, with $K = \tilde{O}(\frac{dm}{n\epsilon^2})$, which implies a sample complexity upper bound of $\tilde{O}(\frac{d^{3/2}}{\alpha \epsilon})$ private data points and $\tilde{O}( d/\alpha)$ public data points. The results improve over the sample complexity bound from  \citet{bassily2018model} in the stronger realizable setting from $\tilde{O}(\frac{d^{3/2}}{\alpha^{3/2}\epsilon})$ and $\tilde{O}(d/\alpha^2)$ to $\tilde{O}(\frac{d^{3/2}}{\alpha \epsilon})$ and $\tilde{O}(d/\alpha)$ respectively in the private and public data.
\end{remark}

\begin{remark}[Optimal choice of $m$]
The upper bound above can be minimized by choosing $m^* = (d^{4-5\tau} n^{2\tau} \epsilon^{2\tau})^\frac{1}{4-2\tau}$. When number of available public data points $m \geq m^*$, then $m$ is not a limiting factor and we should subsample these data points. When $m < m^*$, then $d/m$ is the leading factor, we should use all $m$ data points.
\end{remark}

There are two key observations behind the improvement. First, the teacher classifiers do not have to agree on the labels $y$ as in Lemma~\ref{lem:pigeon_hole}; all they have to do is to agree on something for the majority of the data points. Conveniently, the Tsybakov noise condition implies that the teacher classifiers agree on the Bayes optimal classifier $h^*$. Second, when the teachers agree on $h^*$, the synthetic learning problem with the privately released pseudo-labels is nearly realizable. These intuitions can be formalized with a few lemmas, which will be used in the proof of Theorem~\ref{thm-ug-psq}.

\begin{lemma}[Performance of teacher classifer w.r.t. $h^*$]\label{lem-tnc}
With probability $1-\gamma$ over the training data of $\hat{h}_1, ..., \hat{h}_K$, assume $h^*\in \cH$ is the Bayes optimal classifier and Tsybakov noise condition with parameter $\tau$, then there is a universal constant $C$ such that for all $k=1,2,3,...,K$ 
\begin{align*}
\Dis(\hat{h}_k,h^*) \leq C  
\Big(\frac{dK\log(n/d)+ \log(K/\gamma)}{n}\Big)^\frac{\tau}{2-\tau}.
\end{align*}
\end{lemma}
\begin{proof}
By the equivalent definition of the Tsybakov noise condition and then the learning bounds under TNC (Lemma~\ref{lem-tnc-excess}),
\begin{align*}
\Dis(\hat{h}_k,h^*) \leq \eta(\mathtt{Err}(\hat{h}_k, h^*) - \mathtt{Err}(h^*))^\tau \leq C\Big(\frac{dK\log(n/d)+ \log(K/\gamma)}{n}\Big)^\frac{\tau}{2-\tau}.
\end{align*}
\end{proof}

\begin{lemma}[Total number of mistakes made by one teacher]\label{lem-one-teacher}
Under the condition of Lemma~\ref{lem-tnc}, with probability $1-\gamma$, for all $k=1,2,...,K$ the total number of mistakes made by one teacher classifier $\hat{h}_k$ with respect to $h^*$ can be bounded as:
\begin{align*}
\sum_{j=1}^m \1(\hat{h}_k(x_j) \neq h^*(x_j)) \leq O \bigg( \max \bigg\{m \Dis(\hat{h}_k, h^*), \log \Big(\frac{K}{\gamma}\Big) \bigg\}\bigg).
\end{align*}
\end{lemma}
\begin{proof}
Number of mistakes made by $\hat{h}_k$ with respect to $h^*$ is the empirical disagreement between $\hat{h}_k$ and $h^*$ on $m$ data points, therefore, by Bernstein's inequality (Lemma \ref{lem-gen-bern}),
\begin{align*}
\sum_{j=1}^m \1(\hat{h}_k(x_j) \neq h^*(x_j))\leq & \ O \bigg(m \Dis(\hat{h}_k,h^*) + \sqrt{m \Dis(\hat{h}_k, h^*)\log \Big(\frac{K}{\gamma}\Big)} + \log \Big(\frac{K}{\gamma}\Big) \bigg)\\
\leq & \ O \bigg( \max \bigg\{m \Dis(\hat{h}_k, h^*), \log \Big(\frac{K}{\gamma}\Big) \bigg\}\bigg).
\end{align*}
\end{proof}

Using the above two lemmas we establish a bound on the number of examples where the differentially privately released labels differ from the prediction of $h^*$.
\begin{lemma}[Total queries and cut-off budget]\label{lem-t}
	Let  Algorithm~\ref{alg-psq} be run with the number of teachers $K$ and the cut-off parameter $T$ chosen according to Theorem~\ref{thm-ug-psq}. Assume the conditions of Lemma~\ref{lem-tnc}. 
Then with high probability ($\geq 1-\beta$ over the random coins of Algorithm~\ref{alg-psq} alone and conditioning on the high probability events of Lemma~\ref{lem-tnc} and Lemma~\ref{lem-one-teacher} ), Algorithm~\ref{alg-psq} finishes all $m$ queries without exhausting the cut-off budget and that 
	$$
	\sum_{j=1}^m  \1(\hat{h}^\mathtt{priv}_j \neq h^*(x_j)) \leq T.
	$$
	The $\tilde{O}$ notation in the choice of $K$ and $T$ hides polynomial factors of $\log(K/\gamma),\log(m/\beta)$ where $\gamma$ is from Lemma~\ref{lem-tnc} and ~\ref{lem-one-teacher}.
\end{lemma}
\begin{proof}
	Denote the bound from Lemma \ref{lem-one-teacher} by $B$. By the same Pigeon hole principle argument as in Lemma \ref{lem:pigeon_hole} (but with $y$ replaced by $h^*$), we have that the number of queries that have margin smaller than $K/6$ is at most $3B = O( \max \{m \Dis(\hat{h}_k, h^*), \log(K/\gamma) \})$.   The choice of $K$ ensures that with high probability, over the Laplace random variables in  Algorithm~\ref{alg-priv-agg}, in at least $m-3B$ queries where the answer $\hat{y}_j = h^*(x_j)$, i.e., 
	$$
	\sum_{j=1}^m  \1(\hat{h}^\mathtt{priv}_j \neq h^*(x_j)) \leq 3B := T.
	$$
%
\end{proof}


Now we are ready to put everything together and prove Theorem~\ref{thm-ug-psq}.
\begin{proof}[Proof of Theorem~\ref{thm-ug-psq}]
Denote $\tilde{h} = \argmin_{h \in \cH} \widehat{\Dis}(h,h^*)$ where $\widehat{\Dis}$ is the empirical average of the disagreements over the data points that students have\footnote{Note that in this case we could take $\tilde{h}=h^*$ since $h^*\in\cH$. We are defining this more generally so later we can substitute $h^*$ with other label vector that are not necessarily generated by any hypothesis in $\cH$.}.  By the triangular inequality of the $0-1$ error, 
\begin{align}
\mathtt{Err}(\hat{h}^S) - \mathtt{Err}(h^*)&\leq \Dis(\hat{h}^S,h^*)\nonumber\\
&\leq \widehat{\Dis}(\hat{h}^S, h^*) + 2\sqrt{\frac{(d + \log(4/\gamma))\widehat{\Dis}(\hat{h}^S,h^*)}{m}} + \frac{4(d + \log(4/\gamma))}{m}\nonumber\\
&\leq 2 \widehat{\Dis}(\hat{h}^S,h^*) + \frac{5(d + \log(4/\gamma))}{m}, \label{eq:tnc_thm_deriv1}
\end{align}
where the second line follows from  the uniform Bernstein's inequality  --- apply the first statement Lemma \ref{lem-dis-emp} in Appendix \ref{sec-lemma} with $z = h^*(x)$ and the third line is due to $a + 2\sqrt{ab} + b \leq 2a + 2b$ for non-negative $a,b$.

By the triangular inequality, we have $\widehat{\Dis}(\hat{h}^S,h^*) \leq \widehat{\Dis}(\hat{h}^S,\hat{h}^\mathtt{priv}) + \widehat{\Dis}(\hat{h}^\mathtt{priv},h^*)$, therefore
\begin{align*}
\eqref{eq:tnc_thm_deriv1}
&\leq 2 \widehat{\Dis}(\hat{h}^S,\hat{h}^\mathtt{priv}) + 2\widehat{\Dis}(\hat{h}^\mathtt{priv},h^*) + \frac{5(d + \log(4/\gamma))}{m}\\  
&\leq  2 \widehat{\Dis}(\tilde{h},\hat{h}^\mathtt{priv}) + 2\widehat{\Dis}(\hat{h}^\mathtt{priv},h^*)  + \frac{5(d + \log(4/\gamma))}{m}\\
&\leq 2 \widehat{\Dis}(\tilde{h},h^*) + 4\widehat{\Dis}(\hat{h}^\mathtt{priv},h^*)+ \frac{5(d + \log(4/\gamma))}{m}\\
&= 4\widehat{\Dis}(\hat{h}^\mathtt{priv},h^*)+ \frac{5(d + \log(4/\gamma))}{m}.
\end{align*}
In the second line, we applied the fact that $\hat{h}^S$ is the minimizer of $\widehat{\Dis}(h,\hat{h}^\mathtt{priv})$; in the third line, we applied triangular inequality again and the last line is true because $\widehat{\Dis}(\tilde{h},h^*)=0$ since $\tilde{h} $ is the minimizer and that $h^* \in \cH$. 

Recall that $T$ is the unstable cutoff in Algorithm \ref{alg-psq}. The proof completes by invoking Lemma \ref{lem-t} which shows that the choice of $T$ is appropriate such that $\widehat{\Dis}(\hat{h}^\mathtt{priv},h^*)\leq T/m$ with high probability.
%
%
\end{proof}

In the light of the above analysis, it is clear that the improvement from our analysis under TNC are two-folds: (1) We worked with the disagreement with respect to $h^*$ rather than $y$. (2) We used a uniform Bernstein bound rather than a uniform Hoeffding bound that leads to the faster rate in terms of the number of public data points needed.

\begin{remark}[Reduction to ERM]\label{rmk:reduction1}
	The main challenge in the proof is to appropriately take care of $\hat{h}^\mathtt{priv}$. Although we are denoting it as a classifier, it is in fact a vector that is defined only on $x_1,...,x_m$ rather than a general classifier that can take any input $x$. Since we are using the SVT-based Algorithm~\ref{alg-priv-agg}, $\hat{h}^\mathtt{priv}$ is only well-defined for the student dataset.  Moreover, these privately released ``pseudo-labels'' are not independent, which makes it infeasible to invoke a generic learning bound such as Lemma~\ref{lem-excess}. Our solution is to work with the empirical risk minimizer (ERM, rather than a generic PAC learner as a blackbox) and use uniform convergence (Lemma~\ref{lem-dis-emp}) directly. This is without loss of generality because all learnable problems are learnable by (asymptotic) ERM  \citep{vapnik1995nature,shalev2010learnability}. 
\end{remark}

\subsection{Challenges and New Bounds under Agnostic Setting} \label{sec-main_results_agnostic}
In this section, we present a more refined analysis of the agnostic setting. We first argue that agnostic learning with Algorithm~\ref{alg-psq} will not be consistent in general and competing against the best classifier in $\cH$ seems not the right comparator.  The form of the pseudo-labels mandate that $\hat{h}^S$ is aiming to fit a labeling function that is inherently a voting classifier. The literature on ensemble methods has taught us that the voting classifier is qualitatively different from the individual voters. In particular, the error rate of majority voting classifier can be significantly better, about the same, or significantly worse than the average error rate of the individual voters. We illustrate this with two examples. 

\begin{example}[Voting fails]\label{exp:voting_fail}
	Consider a uniform distribution on $\cX = \{x_1,x_2,x_3,x_4\}$ and that the corresponding label $\P(y=1) = 1$.  Let the hypothesis class be $\cH = \{h_1,h_2,h_3\}$ whose evaluation on $\cX$ are given in Figure~\ref{fig:voting_fail}. Check that the classification error of all three classifiers is $0.5$. Also note that the empirical risk minimizer $\hat{h}$ will be a uniform distribution over $h_1,h_2,h_3$. The majority voting classifiers, learned with iid data sets, will perform significantly worse and converge to a classification error of $0.75$ exponentially quickly as the number of classifiers $K$ goes to $\infty$.
	\begin{figure}[!htbp]
		\centering
			\begin{tabular}{c|cccc|c}
			&$x_1$&$x_2$&$x_3$&$x_4$ & Error\\
			$y$ & 1& 1& 1& 1 &0\\ 
			\hline
			$h_1$&1 &1& 0& 0 &0.5\\
			$h_2$&1&0&1&0 & 0.5\\
			$h_3$&1&0&0&1 & 0.5\\\hline
			$\hat{h}^{\mathtt{agg}}$&1&0&0&0 & 0.75
		\end{tabular}
	\caption{An example where majority voting classifier is significantly worse than the best classifier in $\cH$.}\label{fig:voting_fail}
	\end{figure}
\end{example}
This example illustrates that the PATE framework cannot consistently learn a VC-class in the agnostic setting in general. 
On a positive note, there are also cases where the majority voting classifier boosts the classification accuracy significantly, such as the following example.
\begin{example}[Voting wins]\label{exp:voting_win}
If $\P[\hat{h}(x) \neq y | x] \leq 0.5 - \xi$, where $\xi$ is a small constant, for all $x\in\cX$, then by Hoeffding's inequality,
\begin{align*}
\P [\hat{h}^{\mathtt{agg}}(x)\neq y | x]  = \P \Big[\sum_{k=1}^K \1(\hat{h}_k(x) \neq y) \geq \frac{k}{2} \Big| x \Big]  \leq e^{-2K\xi^2}.
\end{align*}
Thus the error goes to $0$ exponentially as $K\rightarrow \infty$.
\end{example}
These cases call for an alternative distribution-dependent theory of learning that characterizes the performance of Algorithm~\ref{alg-psq} more accurately. 

Next, we propose two changes to the learning paradigms. First, we need to go beyond $\cH$ and compare with the following infinite ensemble classifier
\begin{align*}
h^\mathtt{agg}_{\infty}(x) := \1 \bigg(\E \Big[\frac{1}{K}\sum_{k=1}^k\hat{h}_k(x)\Big|x\Big] \geq \frac{1}{2} \bigg)=  \1\Big( \E[\hat{h}_1(x)|x] \geq \frac{1}{2}\Big).
\end{align*}
The classifier outputs the majority voting result of infinitely many independent teachers, each trained on $n/K$ i.i.d. data points. As discussed earlier, this classifier can be better or worse than a single classifier $\hat{h}_1$ that takes $n/K$ data points, $\hat{h}$ that trains on all $n$ data points or $h^*$ that is the optimal classifier in $\cH$. Note that this classifier also changes as $n/K$ gets larger.

Considering different centers for teacher classifiers to agree on is one of the key ideas of this paper. Figure \ref{fig-centers} shows three kinds of centers for teachers $\hat{h}_1, \hat{h}_2, ..., \hat{h}_9$ to agree on. In \citet{bassily2018model}, the center is the true label $y$ in the realizable setting. In Section \ref{sec-main_results_tnc} under TNC, we analyze the performance of PATE-PSQ, where the center is the best hypothesis $h^*$. Now we are interested in the new center $h^\mathtt{agg}_\infty$ for teachers to agree on.
\begin{figure}[!htbp]
	\centering
	\begin{minipage}{0.32\linewidth}\centering
		\includegraphics[width=0.5\linewidth]{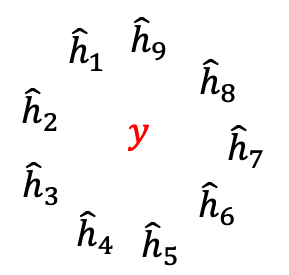}\\
		\small{(a) True label $y$ is the center for in realizable setting.}
	\end{minipage}
	\begin{minipage}{0.32\linewidth}\centering
		\includegraphics[width=0.5\linewidth]{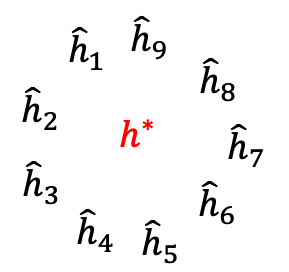}\\
		\small{(b) Best hypothesis $h^*$ is the center under TNC.}
	\end{minipage}
	\begin{minipage}{0.32\linewidth}\centering
		\includegraphics[width=0.5\linewidth]{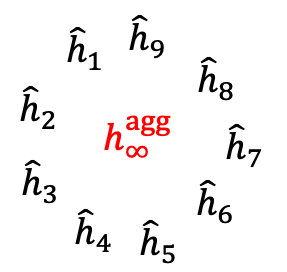}\\
		\small{(c) $h^\mathtt{agg}_\infty$ is our new construction for agnostic setting.}
	\end{minipage}
	\caption{Centers for teachers $\hat{h}_1, \hat{h}_2, ..., \hat{h}_9$ to agree on.\label{fig-centers}}
\end{figure}


Second, we define the \emph{expected margin} for a classifier $\hat{h}_1$ trained on $n$ i.i.d. samples to be
\begin{align}\label{eq:exp_margin}
\Delta_n (x) :=  \Big|\E[ \hat{h}_1(x) | x] - \frac{1}{2} \Big|.
\end{align} 
This quantity captures for a fixed $x\in\cX$, how likely the teachers will agree. For a fixed learning problem $\cH,\cD$ and the number of i.i.d. data points $\hat{h}_1$ is trained upon, the expected margin is a function of $x$ alone. 
The larger $\Delta_{n/K}(x)$ is, the more likely that the ensemble of $K$ teachers agree on a prediction in $\cY$ with high-confidence.
Note that unlike in Example~\ref{exp:voting_win}, we do not require the teachers to agree on $y$. Instead, it measures the extent to which they agree on $h^\mathtt{agg}_{\infty}(x)$, which could be any label.

When the expected margin is bounded away from $0$ for $x$, then the voting classifier outputs $h^\mathtt{agg}_{\infty}(x)$ with probability converging exponentially to $1$ as $K$ gets larger. On the technical level, this definition allows us to \emph{decouple} the stability analysis and accuracy of PATE as the latter relies on how good $h^\mathtt{agg}_{\infty}$ is.

\begin{definition}[Approximate high margin]\label{def-app_high_margin}
	We say that a learning problem with $n$ i.i.d. samples satisfy $(\nu,\xi)$-approximate high-margin condition if 
	$
\P_{x\sim\cD}[ \Delta_n(x) > \xi ] \leq \nu.
	$
\end{definition}
This definition says that with high probability, except for $O(\nu m)$ data points, all other data points in the public dataset have an expected margin of at least $\xi$. Observe that every learning problem has $\xi$ that increases from $0$ to $0.5$ as we vary $\nu$ from $0$ to $1$. The realizability assumption and the Tsybakov noise condition that we considered up to this point imply upper bounds of $\nu$ at fixed $\xi$ (see more details in Remark~\ref{rmk:high-margin-of-tnc}). In Appendix \ref{sec-adult-dataset}, we demonstrate that for the problem of linear classification on Aadult dataset --- clearly an agnostic learning problem --- $(\nu,\xi)$-approximate high margin condition is satisfied with a small $\nu$ and large $\xi$.

The following proposition shows that when a problem is approximate high-margin, there are choices $T$ and $K$ under which the SVT-based PATE provably labels almost all data points with the output of $h^\mathtt{agg}_{\infty}$.
\begin{proposition}\label{prop:SVT_works}
	Assume the learning problem with $n/K$ i.i.d. data points satisfies $(\nu,\xi)$-approximate high-margin condition.  Let Algorithm~\ref{alg-priv-agg} be instantiated with parameters
	$$T \geq  \nu m +  \sqrt{2\nu m \log \Big(\frac{3}{\gamma}\Big)} + \frac{2}{3}\log \Big(\frac{3}{\gamma}\Big),$$
	$$K \geq \max \Big\{ \frac{2\log(3m/\gamma)}{\xi^2}, \frac{3\lambda \big(\log (4m /\delta) +\log(3m/\gamma)\big)}{\xi}\Big\},\footnote{$\lambda = (\sqrt{2T(\epsilon + \log(2/\delta))} + \sqrt{2T\log(2/\delta)})/\epsilon$ according to Algorithm~\ref{alg-priv-agg}.}$$
	then with high probability (over the randomness of the $n$ i.i.d. samples of the private dataset, $m$ i.i.d. samples of the public dataset, and that of the randomized algorithm), Algorithm~\ref{alg-priv-agg} finishes all $m$ rounds and the output is the same as $h^\mathtt{agg}_{\infty}(x_i)$ for all but $T$ of the $i\in[m]$.
\end{proposition}

This proposition provides the utility guarantee to Algorithm~\ref{alg-priv-agg} and generalizes Lemma~\ref{lem-t} from fixing $\xi = 1/6$ into allowing much smaller $\xi$ at a cost of increasing $\nu$. 

Next, we state the learning bounds under the approximate-high margin condition.
\begin{theorem}\label{thm-agree-infty}
	Assume the learning problem with $n/K$ i.i.d. data points satisfies $(\nu,\xi)$-approximate high-margin condition and let $K,T$ be chosen according to Proposition~\ref{prop:SVT_works}, furthermore assume that the privacy parameter of choice $\epsilon \leq \log(2/\delta)$, then
	the output classifier $\hat{h}^S$ of Algorithm \ref{alg-psq} in the agnostic setting satisfies that with probability $\geq 1-2\gamma$,
	\begin{align*}
	\mathtt{Err}(\hat{h}^S) - \mathtt{Err}(h^{\mathtt{agg}}_{\infty}) &\leq \min_{h\in\cH}\Dis(h,h^{\mathtt{agg}}_{\infty}) + \frac{2T}{m}+ \tilde{O}\Big(\sqrt{\frac{d}{m}}\Big)\\
	&\leq \min_{h\in\cH}\Dis(h,h^{\mathtt{agg}}_{\infty}) + 2\nu + \tilde{O}\Big(\sqrt{\frac{d}{m}}\Big) .
	\end{align*}
\end{theorem}

The voting classifier $\hat{h}^\mathtt{agg}$ is usually not in the original hypothesis class $\cH$, so we can take a wider view of the hypothesis class and define the voting hypothesis space $\mathrm{Vote}(\cH)$ where the learning problem becomes realizable. Note if the VC dimension of $\cH$ is $d$, then the VC dimension of $\mathrm{Vote}_K(\cH) \leq Kd$.  In practice, this suggests using ensemble methods such as AdaBoost for $K$ iterations.
\begin{theorem}\label{thm-vote-space}
	Under the same assumption of Theorem~\ref{thm-agree-infty}, suppose we train an ensemble classifier within the voting hypothesis space $\mathrm{Vote}_K(\cH)$ in the student domain, then the output classifier $\hat{h}^S$ of Algorithm \ref{alg-psq} in the agnostic setting satisfies that with probability $\geq 1-2\gamma$,
\begin{align*}
	\mathtt{Err}(\hat{h}^S) - \mathtt{Err}(h^{\mathtt{agg}}_{\infty}) \leq \frac{4T}{m} + \frac{5(Kd + \log(4/\gamma))}{m}= \tilde{O}\Big( \nu + \frac{\log(4/\gamma)}{m} + \frac{d \sqrt{\nu}}{\xi \sqrt{m}}\Big).
\end{align*}
%
\end{theorem}

\begin{remark}\label{rmk:high-margin-of-tnc}
Whether the bounds in Theorem~\ref{thm-agree-infty}~and~\ref{thm-vote-space} will vanish as $m,n\rightarrow\infty$ depends strongly on how parameter $\nu$ and $\xi$ change as $n/K$ gets larger. Intuitively, if the learner converges to a single classifier $h^*$, as in the realizable case or under TNC, then we can show that the learning problem satisfy $(\nu,\xi)$-approximate high-margin condition with $\xi = 1/6$ and $\nu \leq \tilde{O}((dK/n)^{\frac{\tau}{2-\tau}})$. Substituting this quantities into Theorem~\ref{thm-agree-infty} and using the fact that $\nu$ also bounds the disagreement between $h^*$ and $h^\mathtt{agg}_\infty$ allows us obtain a bound that vanishes as $n$ gets larger. More generally, in the agnostic case, it is reasonable to assume that the  ``teachers'' will get \emph{more confident} in their individual prediction for most data points as $n/K \rightarrow \infty$. 
We argue this is a more modest requirement than requiring the ``teachers'' to get \emph{more accurate}.
\end{remark}
%

\subsection{PATE with Active Student Queries}  \label{sec-main_results_active}
In previous subsections, we have proved stronger learning bounds for PATE framework under TNC and in agnostic setting. However, all these results are based on the variants of PATE that aim at \emph{passively} releasing \emph{almost all} student queries. In this section we address the following question:

\begin{center}
Can we do even better if we cherry-pick queries to label?
\end{center}
The hope is that this allows us to spend privacy budget only on those queries that add new information for the interest of training a classifier, hence resulting in a more favorable privacy-utility tradeoff.  Without privacy constraints, this problem is known as active learning and it is often possible to save exponentially in the number of labels needed comparing to the passive learning model.

In Algorithm \ref{alg-asq}, we propose a new algorithm called PATE with Active Student Queries (PATE-ASQ) which uses the disagreement-based active learning algorithm (Algorithm \ref{alg-active_learning}) as the subroutine. Then we provide its utility guarantee.
\begin{algorithm}[!htbp]
	\caption{PATE-ASQ}
	\label{alg-asq}
	{\bf Input:}
	Labeled private teacher dataset $D^T$, unlabeled public student dataset $D^S$, privacy parameters $\epsilon, \delta > 0$, number of splits $K$, maximum number of queries $\ell$, failure probability $\gamma$.
	\begin{algorithmic}[1]
	    \STATE Randomly and evenly split the teacher dataset $D^T$ into $K$ parts $D^T_k \subseteq D^T$ where $k \in [K]$
	    \STATE Train $K$ classifiers $\hat{h}_k \in \cH$, one from each part $D^T_k$.
        \STATE Declare ``Labeling Service'' $\leftarrow$ Algorithm~\ref{alg-priv-agg-pate} with $\hat{h}_1,...,\hat{h}_K$,  $\ell$, $\epsilon,\delta$, with an unspecified ``nature''.
        \STATE Initiate an active learning oracle  (e.g., Algorithm~\ref{alg-active_learning}) with an iterator over $D^S$ being the ``data stream'', hypothesis class $\cH$, failure probability $\gamma$. Set the ``labeling service'' to be Algorithm~\ref{alg-priv-agg-pate} with parameter $\hat{h}_1,...,\hat{h}_K$,  $\ell$, $\epsilon,\delta$, and set the ``nature'' to be the ``request for label'' calls in the active learning oracle. 
        \STATE Set $\hat{h}^S$ to be the ``current output'' from active learning oracle.
	\end{algorithmic}
	{\bf Output:} Return $\hat{h}^S$.
\end{algorithm}



\begin{theorem}[Utility guarantee of Algorithm \ref{alg-asq}]\label{thm-ug-asq}
With probability at least $1-\gamma$, there exists universal constants $C_1,C_2$ such that for all
\begin{align*}
\alpha \geq C_1\max \bigg\{\eta^\frac{2}{2-\tau}\Big(\frac{dK\log(n/d)+ \log(2K/\gamma)}{n}\Big)^\frac{\tau}{2-\tau}, \frac{d\log((m+n)/d) + \log(2/\gamma)}{m}\bigg\},
\end{align*}
the output $\hat{h}^S$ of  Algorithm~\ref{alg-asq} with parameter $\ell,K$ satisfying
\begin{align*}
\ell = C_2\theta(\alpha)\Big( 1 + \log \big(\frac{1}{\alpha}\big) \Big) \bigg( d\log(\theta(\alpha)) + \log \Big(\frac{\log(1/\alpha)}{\gamma/2} \Big)\bigg)
\end{align*}
$$K = \frac{6\sqrt{\log(2n)}(\sqrt{\ell \log(1/\delta)} + \sqrt{\ell \log(1/\delta) + \epsilon \ell})}{\epsilon}$$
obeys that 
$$
\mathtt{Err}(\hat{h}^S) - \mathtt{Err}(h^*) \leq \alpha.
$$
Specifically, when we choose
\begin{align*}
\alpha= C_1\max \bigg\{\eta^\frac{2}{2-\tau}\Big(\frac{dK\log(n/d)+ \log(2K/\gamma)}{n}\Big)^\frac{\tau}{2-\tau}, \frac{d\log((m+n)/d) + \log(2/\gamma)}{m}\bigg\},  
\end{align*}
and also $\epsilon\leq \log(1/\delta)$, then it follows that
\begin{align*}
\mathtt{Err}(\hat{h}^S) - \mathtt{Err}(h^*) = \tilde{O}\bigg(\max\Big\{\big(\frac{d^{1.5}\sqrt{\theta(\alpha)\log(1/\delta)}}{n\epsilon}\big)^{\frac{\tau}{2-\tau}}, \frac{d}{m}\Big\}\bigg),
\end{align*}
where $\tilde{O}$ hides logarithmic factors in $m,n,1/\gamma$.
\end{theorem}

\begin{remark}
The bound above resembles the learning bound we obtain using the passive student queries with Algorithm~\ref{alg-priv-agg} as the privacy procedure, except for the additional dependence on the disagreement coefficients. Interestingly, active learning achieves this bound without using the sophisticated (and often not practical) algorithmic components from DP, such as sparse sector technique to save privacy losses. Instead, we can get away with using simple Gaussian mechanism as in Algorithm~\ref{alg-priv-agg-pate}.
\end{remark}


\begin{remark}[Blackbox reduction, revisited]
In contrary to our discussion in Remark~\ref{rmk:reduction1}, notice that we are using Algorithm~\ref{alg-priv-agg-pate} instead of Algorithm~\ref{alg-priv-agg} as the labeling services, which allows us to reduce to any learner as a blackbox. This makes it possible to state formally results even for deep neural networks or other family of methods where obtaining ERM is hard but learning is conjectured to be easy in theory and in practice.
\end{remark}

\begin{remark}[Relationships between SVT and active learning]
	There is an intriguing analogy between the Algorithm~\ref{alg-priv-agg} which simply labels all queries with an advanced DP mechanism and Algorithm~\ref{alg-priv-agg-pate} which uses active learning with a simple DP mechanism. On a high level, both approaches are doing selection. Active learning selects those queries that are near the decision boundary to be informative for learning; the sparse-vector-technique approach \emph{essentially} selects those queries that are not stable to spend privacy budget on.
	
	One curious question is whether the two sets of selected data points are substantially overlapping. If not, then we might be able to combine the two and achieve even better private-utility tradeoff.
\end{remark}

\section{Experiments}
In this section, we present our empirical studies of PATE-PSQ and PATE-ASQ algorithms. Section \ref{sec:exp-set} describes how we set up our experiments, and Section \ref{sec:exp-result} show our results.

\subsection{Experimental Settings}\label{sec:exp-set}

\paragraph{Algorithms compared.} 
We focus on comparing the classification accuracy of the passive and active learning versions of PATE on a holdout test set (``Utility'') when both algorithms are calibrated to the same privacy budget $\epsilon$ (``privacy''). 
To set baselines, we also compare them with non-private versions of them (no noise added to the votes, or $\epsilon = +\infty$), denoted by PATE-PSQ-NP and PATE-ASQ-NP. We remark that the PATE-PSQ we implement is the Gaussian mechanism version \citep{papernot2018scalable}. While we have shown that it has higher \emph{asymptotic} sample complexity comparing to the more advanced version based on SVT \citep{bassily2018model} (Section~\ref{sec-main_results_tnc}), we found that the Gaussian mechanism version performs better for the realistically-sized datasets that we considered.
Linear models are used for all of these algorithms for simplicity.
For active learning, we follow the practical implementation of the disagreement-based active learning by \citet{yan2018active}, which does not require the learner to explicitly maintain the (exponentially large) region of disagreement.


\paragraph{Datasets.} We do our experiments on three binary classification datasets, mushroom, a9a, and real-sim. All of them are obtained from LIBSVM dataset website \footnote{\url{https://www.csie.ntu.edu.tw/~cjlin/libsvmtools/datasets/}}. See Table \ref{tab:dataset} for the statistics of them. If a dataset had been previously split into training and testing parts, we combine them together and record the total number of all data points. For all datasets, $80\%$ of all data points are randomly selected to be considered private and used to train teacher classifiers. $2\%$ of all data points are randomly selected as public student unlabeled data points. The remaining $18\%$ data points are reserved for testing. We repeat these random selection processes for $30$ times. 

\begin{table}[!htbp]
\centering
\caption{Statistics of datasets.}\label{tab:dataset}
\resizebox{\textwidth}{!}{
\begin{tabular}{ccccccc}
\noalign{\smallskip} \hline \noalign{\smallskip}
\textbf{Dataset}  & \textbf{\# All} & \textbf{\# Train} & \textbf{\# Unlabeled} & \textbf{\# Budget} & \textbf{\# Test} & \textbf{\# Dimension} \\ \noalign{\smallskip} \hline \noalign{\smallskip}
mushroom &  $8,124$            &  $6,499$        & $163$    & $49$            & $1,462$            & $112$     \\ \noalign{\smallskip} \hline \noalign{\smallskip}
a9a      &  $48,842$            & $39,073$         &   $977$   & $293$              &  $8,792$           & $123$     \\ \noalign{\smallskip} \hline \noalign{\smallskip}
real-sim      & $72,309$         & $57,847$     & $1,447$  & $434$             & $13,015$       & $20,958$      \\ \noalign{\smallskip} \hline \noalign{\smallskip}
\end{tabular}}
\end{table}

\paragraph{Parameter settings.} Number of teachers $K$ is set on all datasets so that each teacher classifier gets trained with approximately $100$ data points. $30\%$ of student unlabeled data points are set as the total budget of queries for PATE-ASQ and PATE-ASQ-NP. See Table \ref{tab:dataset}. $\epsilon=0.5, 1.0, 2.0$ and $\delta=1/n$ are set as privacy parameters for all datasets, where $n$ is number of private teacher data points. All privacy accounting and calibration are conducted via AutoDP \citep{wang2018subsample}, and the tight analytical calibration and composition of Gaussian mechanisms are due to \citep{balle2018improving}.


\paragraph{Privacy loss vs. privacy budget.} Besides the privacy budget parameter $\epsilon$ that the algorithms receive as an input, it is often the case that the active learning algorithm halts before exhausting the query budget of (30\% of the total number of unlabeled data points). Therefore the privacy loss incurred after running PATE-ASQ might be smaller than the prescribed privacy budget. We refer to the privacy loss $\epsilon_\mathtt{ex ~post}$, since it is determined by the output.

\begin{table}[t]
\centering
\caption{Utility and privacy results of different PATE models. \textbf{\# Queries} shows the number of queries actually answered in experiments. \textbf{Accuracy} is reported as $\mathtt{mean} \pm 1.96 \times \mathtt{standard\_error}/\sqrt{30}$, i.e., $98\%$ asymptotic confidence interval of the expected accuracy based on inverting Wald's test. All ``PATE-'' prefixes of methods are omitted to improve readability.} \label{tab:result}
\label{tab:exp}
\begin{tabular}{cccccc}
\noalign{\smallskip} \hline \noalign{\smallskip}
\textbf{Dataset}                   & \textbf{Method}   & \textbf{\# Queries} & $\epsilon$ & $\epsilon_\texttt{ex~post}$ & \textbf{Accuracy}               \\ \noalign{\smallskip} \hline \noalign{\smallskip}
\multirow{8}{*}{mushroom} & PSQ-NP & $163$        & $+\infty$   & $+\infty$ & $\mathbf{0.9773}\pm0.0006$ \\ 
                          & ASQ-NP & $47.3 \pm 0.2$        & $+\infty$   & $+\infty$ &  $0.9146\pm0.0036$                    \\
                          & PSQ & $163$        & $0.5$   & $0.5$ &   $0.6416\pm0.0036$                   \\
                          & ASQ & $40.1 \pm 0.7$        & $0.5$   & $\mathbf{0.4461}$ & $\mathbf{0.6418}\pm 0.0091$
                          \\
                          & PSQ & $163$        & $1.0$   & $1.0$ &   $0.7534\pm0.0045$                   \\
                          & ASQ & $42.9 \pm 0.5$        & $1.0$   & $\mathbf{0.9267}$ & $\mathbf{0.7727}\pm 0.0098$\\
                          & PSQ & $163$        & $2.0$   & $2.0$ &   $\mathbf{0.8974}\pm0.0027$                   \\
                          & ASQ & $46.5 \pm 0.3$        & $2.0$   & $\mathbf{1.9410}$ & $0.8858\pm 0.0059$\\\noalign{\smallskip} \hline \noalign{\smallskip}
\multirow{8}{*}{a9a}      & PSQ-NP & $977$        & $+\infty$  & $+\infty$ & $\mathbf{0.5555}\pm0.0157$ \\ 
                          & ASQ-NP & $225.6\pm5.0$        & $+\infty$   & $+\infty$& $0.5461\pm0.0160$                      \\
                          & PSQ & $977$        & $0.5$   & $0.5$ & $0.5040\pm0.0034$                      \\
                          & ASQ & $293$        & $0.5$   & $0.5$ & $\mathbf{0.5212}\pm0.0088$ \\
                          & PSQ & $977$        & $1.0$   & $1.0$ &   $0.5171\pm0.0050$                   \\
                          & ASQ & $290.8 \pm 0.8$        & $1.0$   & $\mathbf{0.9958}$ & $\mathbf{0.5369}\pm 0.0103$ \\
                          & PSQ & $977$        & $2.0$   & $2.0$ &   $0.5176\pm0.0070$                   \\
                          & ASQ & $290.3 \pm 0.9$        & $2.0$   & $\mathbf{1.9896}$ & $\mathbf{0.5543}\pm 0.0089$\\\noalign{\smallskip} \hline \noalign{\smallskip}
\multirow{8}{*}{real-sim}      & PSQ-NP & $1,447$       & $+\infty$   & $+\infty$ & $0.8234\pm0.0014$ \\ 
                          & ASQ-NP & $434$       & $+\infty$   & $+\infty$ & $\mathbf{0.8289}\pm0.0008$                     \\
                          & PSQ & $1,447$       & $0.5$   & $0.5$ & $0.6355\pm0.0065$                      \\
                          & ASQ & $434$       & $0.5$   & $0.5$ & $\mathbf{0.7389}\pm0.0014$ \\
                          & PSQ & $1,447$        & $1.0$   & $1.0$ &   $0.7550\pm0.0058$                   \\
                          & ASQ & $434$        & $1.0$   & $1.0$ & $\mathbf{0.8040}\pm 0.0009$            \\
                          & PSQ & $1,447$        & $2.0$   & $2.0$ &   $0.8025\pm0.0037$                   \\
                          & ASQ & $434$        & $2.0$   & $2.0$ & $\mathbf{0.8231}\pm 0.0009$\\\noalign{\smallskip} \hline \noalign{\smallskip}
\end{tabular}
\end{table}

\subsection{Experimental Results}\label{sec:exp-result}
The results are presented in Table \ref{tab:result}, where both utility (classification accuracy on the test set) and privacy (privacy budget $\epsilon$ and privacy loss $\epsilon_{\text{ex post}}$) metrics are reported. Best results in each category are marked in bold fonts. We make a few observations of the results below.
\begin{enumerate}
    \item Given the same privacy budget, ASQ performs substantially better than PSQ in most cases. The improvement is sometimes $10\%$ (real-sim / $\epsilon=0.5$). The only exception is when $\epsilon=2.0$ on the ``mushroom'' dataset, in which the active learning performed substantially worse than the passive-learning counterpart in the non-private baseline as well.
    \item ASQ incurs a smaller private loss $\epsilon_\mathtt{ex~post}$ than PSQ, due to possibly fewer queries being selected by the active learning algorithm than the pre-specified query budget.
    \item As $\epsilon$ increases, less noise is injected by the Gaussian mechanisms, and the performance improves for both PSQ and ASQ. In the regime of small $\epsilon$ (stronger privacy), we often see a greater improvement in ASQ.
    \item ASQ requires privately releasing a much smaller number of labels
    while maintaining comparable performances as PSQ. Although ASQ algorithms use up all labeling budget on real-sim datasets, ASQ algorithms do not run out of them on mushroom and a9a datasets in most cases. 
    \item ASQ-NP does not always perform better than PSQ-NP algorithms, which meets our understanding from active learning literature. It only performs better than PSQ-NP on real-sim datasets.
\end{enumerate}
\section{Conclusion}
Existing theoretical analysis shows that PATE framework consistently learns any VC-classes in the realizable setting, but not in the more general cases. We show that PATE learns any VC-classes under Tsybakov noise condition (TNC) with fast rates. When specializing to the realizable case, our results improve the best known sample complexity bound for both the public and private data. We show that PATE is incompatible with the agnostic learning setting because it is essentially trying to learn a different class of voting classifiers which could be better, worse, or comparable to the best classifier in the base-class. Lastly, we investigated the PATE framework with active learning and showed that simple Gaussian mechanism suffices for obtaining the same fast rates under TNC. In addition, our experiments on PATE-ASQ show it works as an efficient algorithm in practice.

Future work includes understanding different selections made by sparse vector technique and active learning, as well as addressing the open theoretical problem \emph{at large} --- developing ERM-oracle efficient algorithm for the private agnostic learning when a public unlabeled dataset is available.

\subsection*{Acknowledgments}
The work is partially supported by NSF CAREER Award \#2048091 and generous gifts from NEC Labs, Google and Amazon Web Services. The authors would like to thank Songbai Yan for sharing their code with a practical implementation of the disagreement-based active learning algorithms from \cite{yan2018active}.

\bibliography{asq}

\begin{thebibliography}{41}
\providecommand{\natexlab}[1]{#1}
\providecommand{\url}[1]{\texttt{#1}}
\expandafter\ifx\csname urlstyle\endcsname\relax
  \providecommand{\doi}[1]{doi: #1}\else
  \providecommand{\doi}{doi: \begingroup \urlstyle{rm}\Url}\fi

\bibitem[Abadi et~al.(2016)Abadi, Chu, Goodfellow, McMahan, Mironov, Talwar,
  and Zhang]{abadi2016deep}
Martin Abadi, Andy Chu, Ian Goodfellow, H~Brendan McMahan, Ilya Mironov, Kunal
  Talwar, and Li~Zhang.
\newblock Deep learning with differential privacy.
\newblock In \emph{Conference on Computer and Communications Security
  (CCS-16)}, pages 308--318, 2016.

\bibitem[Alon et~al.(2019)Alon, Bassily, and Moran]{bassily2019limits}
Noga Alon, Raef Bassily, and Shay Moran.
\newblock Limits of private learning with access to public data.
\newblock In \emph{Neural Information Processing Systems (NeurIPS-19)}, pages
  10342--10352, 2019.

\bibitem[Balle and Wang(2018)]{balle2018improving}
Borja Balle and Yu-Xiang Wang.
\newblock Improving gaussian mechanism for differential privacy: Analytical
  calibration and optimal denoising.
\newblock In \emph{International Conference in Machine Learning (ICML-18)},
  2018.

\bibitem[Bassily et~al.(2014)Bassily, Smith, and Thakurta]{bassily2014private}
Raef Bassily, Adam Smith, and Abhradeep Thakurta.
\newblock Private empirical risk minimization: Efficient algorithms and tight
  error bounds.
\newblock In \emph{Symposium on Foundations of Computer Science (FOCS-14)},
  pages 464--473, 2014.

\bibitem[Bassily et~al.(2018{\natexlab{a}})Bassily, Thakkar, and
  Thakurta]{bassily2018arxiv}
Raef Bassily, Om~Thakkar, and Abhradeep Thakurta.
\newblock Model-agnostic private learning via stability.
\newblock \emph{arXiv preprint arXiv:1803.05101}, 2018{\natexlab{a}}.

\bibitem[Bassily et~al.(2018{\natexlab{b}})Bassily, Thakkar, and
  Thakurta]{bassily2018model}
Raef Bassily, Om~Thakkar, and Abhradeep~Guha Thakurta.
\newblock Model-agnostic private learning.
\newblock In \emph{Neural Information Processing Systems (NeurIPS-18)}, pages
  7102--7112, 2018{\natexlab{b}}.

\bibitem[Beimel et~al.(2013)Beimel, Nissim, and
  Stemmer]{beimel2013characterizing}
Amos Beimel, Kobbi Nissim, and Uri Stemmer.
\newblock Characterizing the sample complexity of private learners.
\newblock In \emph{Innovations in Theoretical Computer Science Conference
  (ITCS-13)}, pages 97--110, 2013.

\bibitem[Beimel et~al.(2016)Beimel, Nissim, and Stemmer]{beimel2016private}
Amos Beimel, Kobbi Nissim, and Uri Stemmer.
\newblock Private learning and sanitization: Pure vs. approximate differential
  privacy.
\newblock \emph{Theory of Computing}, 12\penalty0 (890):\penalty0 1--61, 2016.

\bibitem[Boucheron et~al.(2005)Boucheron, Bousquet, and
  Lugosi]{boucheron2005theory}
St{\'e}phane Boucheron, Olivier Bousquet, and G{\'a}bor Lugosi.
\newblock Theory of classification: A survey of some recent advances.
\newblock \emph{ESAIM: probability and statistics}, 9:\penalty0 323--375, 2005.

\bibitem[Bousquet et~al.(2004)Bousquet, Boucheron, and
  Lugosi]{bousquet2004introduction}
Olivier Bousquet, St{\'e}phane Boucheron, and G{\'a}bor Lugosi.
\newblock Introduction to statistical learning theory.
\newblock \emph{Advanced Lectures on Machine Learning: ML Summer Schools},
  pages 169--207, 2004.

\bibitem[Bun and Steinke(2016)]{bun2016concentrated}
Mark Bun and Thomas Steinke.
\newblock Concentrated differential privacy: Simplifications, extensions, and
  lower bounds.
\newblock In \emph{Theory of Cryptography Conference (TCC-16)}, pages 635--658,
  2016.

\bibitem[Bun et~al.(2015)Bun, Nissim, Stemmer, and
  Vadhan]{bun2015differentially}
Mark Bun, Kobbi Nissim, Uri Stemmer, and Salil Vadhan.
\newblock Differentially private release and learning of threshold functions.
\newblock In \emph{Symposium on Foundations of Computer Science (FOCS-15)},
  pages 634--649, 2015.

\bibitem[Chaudhuri and Hsu(2011)]{chaudhuri2011sample}
Kamalika Chaudhuri and Daniel Hsu.
\newblock Sample complexity bounds for differentially private learning.
\newblock In \emph{Annual Conference on Learning Theory (COLT-11)}, pages
  155--186, 2011.

\bibitem[Chaudhuri et~al.(2011)Chaudhuri, Monteleoni, and
  Sarwate]{chaudhuri2011differentially}
Kamalika Chaudhuri, Claire Monteleoni, and Anand~D Sarwate.
\newblock Differentially private empirical risk minimization.
\newblock \emph{Journal of Machine Learning Research}, 12\penalty0
  (3):\penalty0 1069--1109, 2011.

\bibitem[Dagan and Feldman(2020)]{dagan2020pac}
Yuval Dagan and Vitaly Feldman.
\newblock Pac learning with stable and private predictions.
\newblock In \emph{Annual Conference on Learning Theory (COLT-20)}, pages
  1389--1410, 2020.

\bibitem[Dwork and Feldman(2018)]{dwork2018privacy}
Cynthia Dwork and Vitaly Feldman.
\newblock Privacy-preserving prediction.
\newblock In \emph{Annual Conference on Learning Theory (COLT-18)}, pages
  1693--1702, 2018.

\bibitem[Dwork and Roth(2014)]{dwork2014algorithmic}
Cynthia Dwork and Aaron Roth.
\newblock The algorithmic foundations of differential privacy.
\newblock \emph{Foundations and Trends in Theoretical Computer Science},
  9\penalty0 (3--4):\penalty0 211--407, 2014.

\bibitem[Dwork et~al.(2006)Dwork, McSherry, Nissim, and
  Smith]{dwork2006calibrating}
Cynthia Dwork, Frank McSherry, Kobbi Nissim, and Adam Smith.
\newblock Calibrating noise to sensitivity in private data analysis.
\newblock In \emph{Theory of Cryptography Conference (TCC-06)}, pages 265--284,
  2006.

\bibitem[Dwork et~al.(2010)Dwork, Rothblum, and Vadhan]{dwork2010boosting}
Cynthia Dwork, Guy~N Rothblum, and Salil Vadhan.
\newblock Boosting and differential privacy.
\newblock In \emph{Symposium on Foundations of Computer Science (FOCS-10)},
  pages 51--60, 2010.

\bibitem[Erlingsson et~al.(2014)Erlingsson, Pihur, and
  Korolova]{erlingsson2014rappor}
{\'U}lfar Erlingsson, Vasyl Pihur, and Aleksandra Korolova.
\newblock Rappor: Randomized aggregatable privacy-preserving ordinal response.
\newblock In \emph{Conference on Computer and Communications Security
  (CCS-14)}, pages 1054--1067, 2014.

\bibitem[Hanneke(2014)]{hanneke2014theory}
Steve Hanneke.
\newblock Theory of disagreement-based active learning.
\newblock \emph{Foundations and Trends in Machine Learning}, 7\penalty0
  (2-3):\penalty0 131--309, 2014.

\bibitem[Hardt and Rothblum(2010)]{hardt2010multiplicative}
Moritz Hardt and Guy~N Rothblum.
\newblock A multiplicative weights mechanism for privacy-preserving data
  analysis.
\newblock In \emph{Symposium on Foundations of Computer Science (FOCS-20)},
  pages 61--70, 2010.

\bibitem[Kasiviswanathan et~al.(2011)Kasiviswanathan, Lee, Nissim,
  Raskhodnikova, and Smith]{kasiviswanathan2011can}
Shiva~Prasad Kasiviswanathan, Homin~K Lee, Kobbi Nissim, Sofya Raskhodnikova,
  and Adam Smith.
\newblock What can we learn privately?
\newblock \emph{SIAM Journal on Computing}, 40\penalty0 (3):\penalty0 793--826,
  2011.

\bibitem[Machanavajjhala et~al.(2008)Machanavajjhala, Kifer, Abowd, Gehrke, and
  Vilhuber]{machanavajjhala2008privacy}
Ashwin Machanavajjhala, Daniel Kifer, John Abowd, Johannes Gehrke, and Lars
  Vilhuber.
\newblock Privacy: Theory meets practice on the map.
\newblock In \emph{International Conference on Data Engineering (ICDE-08)},
  pages 277--286, 2008.

\bibitem[Mammen and Tsybakov(1999)]{mammen1999smooth}
Enno Mammen and Alexandre~B Tsybakov.
\newblock Smooth discrimination analysis.
\newblock \emph{The Annals of Statistics}, 27\penalty0 (6):\penalty0
  1808--1829, 1999.

\bibitem[McMahan et~al.(2018)McMahan, Ramage, Talwar, and
  Zhang]{mcmahan2018learning}
H~Brendan McMahan, Daniel Ramage, Kunal Talwar, and Li~Zhang.
\newblock Learning differentially private recurrent language models.
\newblock In \emph{International Conference on Learning Representations
  (ICLR-18)}, 2018.

\bibitem[Nandi and Bassily(2020)]{nandi2020privately}
Anupama Nandi and Raef Bassily.
\newblock Privately answering classification queries in the agnostic pac model.
\newblock In \emph{International Conference on Algorithmic Learning Theory
  (ALT-20)}, pages 687--703, 2020.

\bibitem[Nissim et~al.(2007)Nissim, Raskhodnikova, and Smith]{nissim2007smooth}
Kobbi Nissim, Sofya Raskhodnikova, and Adam Smith.
\newblock Smooth sensitivity and sampling in private data analysis.
\newblock In \emph{Symposium on Theory of Computing (STOC-07)}, pages 75--84,
  2007.

\bibitem[Papernot et~al.(2017)Papernot, Abadi, Erlingsson, Goodfellow, and
  Talwar]{papernot2017semi}
Nicolas Papernot, Mart{\'\i}n Abadi, {\'U}lfar Erlingsson, Ian Goodfellow, and
  Kunal Talwar.
\newblock Semi-supervised knowledge transfer for deep learning from private
  training data.
\newblock In \emph{International Conference on Learning Representations
  (ICLR-17)}, 2017.

\bibitem[Papernot et~al.(2018)Papernot, Song, Mironov, Raghunathan, Talwar, and
  Erlingsson]{papernot2018scalable}
Nicolas Papernot, Shuang Song, Ilya Mironov, Ananth Raghunathan, Kunal Talwar,
  and {\'U}lfar Erlingsson.
\newblock Scalable private learning with pate.
\newblock In \emph{International Conference on Learning Representations
  (ICLR-18)}, 2018.

\bibitem[Shalev-Shwartz et~al.(2010)Shalev-Shwartz, Shamir, Srebro, and
  Sridharan]{shalev2010learnability}
Shai Shalev-Shwartz, Ohad Shamir, Nathan Srebro, and Karthik Sridharan.
\newblock Learnability, stability and uniform convergence.
\newblock \emph{Journal of Machine Learning Research}, 11\penalty0
  (90):\penalty0 2635--2670, 2010.

\bibitem[Shokri and Shmatikov(2015)]{shokri2015privacy}
Reza Shokri and Vitaly Shmatikov.
\newblock Privacy-preserving deep learning.
\newblock In \emph{Conference on Computer and Communications Security
  (CCS-15)}, pages 1310--1321, 2015.

\bibitem[Thakurta and Smith(2013)]{thakurta2013differentially}
Abhradeep~Guha Thakurta and Adam Smith.
\newblock Differentially private feature selection via stability arguments, and
  the robustness of the lasso.
\newblock In \emph{Annual Conference on Learning Theory (COLT-13)}, pages
  819--850, 2013.

\bibitem[Tsybakov(2004)]{tsybakov2004optimal}
Alexander~B Tsybakov.
\newblock Optimal aggregation of classifiers in statistical learning.
\newblock \emph{The Annals of Statistics}, 32\penalty0 (1):\penalty0 135--166,
  2004.

\bibitem[Vapnik(1995)]{vapnik1995nature}
Vladimir~N Vapnik.
\newblock \emph{The nature of statistical learning theory}.
\newblock Springer, 1995.

\bibitem[Wang et~al.(2015)Wang, Fienberg, and Smola]{wang2015privacy}
Yu-Xiang Wang, Stephen Fienberg, and Alex Smola.
\newblock Privacy for free: Posterior sampling and stochastic gradient monte
  carlo.
\newblock In \emph{International Conference on Machine Learning (ICML-15)},
  pages 2493--2502, 2015.

\bibitem[Wang et~al.(2016)Wang, Lei, and Fienberg]{wang2016learning}
Yu-Xiang Wang, Jing Lei, and Stephen~E. Fienberg.
\newblock Learning with differential privacy: Stability, learnability and the
  sufficiency and necessity of erm principle.
\newblock \emph{Journal of Machine Learning Research}, 17\penalty0
  (183):\penalty0 1--40, 2016.

\bibitem[Wang et~al.(2019)Wang, Balle, and Kasiviswanathan]{wang2018subsample}
Yu-Xiang Wang, Borja Balle, and Shiva Kasiviswanathan.
\newblock Subsampled r{\'e}nyi differential privacy and analytical moments
  accountant.
\newblock In \emph{International Conference on Artificial Intelligence and
  Statistics (AISTATS-19)}, 2019.

\bibitem[Yan et~al.(2018)Yan, Chaudhuri, and Javidi]{yan2018active}
Songbai Yan, Kamalika Chaudhuri, and Tara Javidi.
\newblock Active learning with logged data.
\newblock In \emph{International Conference on Machine Learning (ICML-18)},
  pages 5521--5530, 2018.

\bibitem[Zhang and Chaudhuri(2014)]{zhang2014beyond}
Chicheng Zhang and Kamalika Chaudhuri.
\newblock Beyond disagreement-based agnostic active learning.
\newblock In \emph{Neural Information Processing Systems (NeurIPS-14)}, pages
  442--450, 2014.

\bibitem[Zhao et~al.(2019)Zhao, Papernot, Singh, Polyzotis, and
  Odena]{zhao2019improving}
Zhengli Zhao, Nicolas Papernot, Sameer Singh, Neoklis Polyzotis, and Augustus
  Odena.
\newblock Improving differentially private models with active learning.
\newblock \emph{arXiv preprint arXiv:1910.01177}, 2019.

\end{thebibliography}

\newpage
\appendix

\section{Proofs of Stated Technical Results}\label{sec:refined_proofs}

\subsection{Proofs of Existing Results}
In this subsection, we provide the privacy analysis as well as reproving the results of \citet{bassily2018model} in our notation so that it becomes clear where the improvement is coming from.

\begin{theorem}[Restatement of Theorem \ref{thm-pg-agg}]
	Algorithm~\ref{alg-priv-agg-pate}~and~\ref{alg-priv-agg} are both ($\epsilon,\delta$)-DP.
\end{theorem}

The proof for Algorithm~\ref{alg-priv-agg-pate} follows straightforwardly from Gaussian mechanism because the number of ``teachers'' who predict $1$ will have a global sensitivity of $1$. The proof for Algorithm~\ref{alg-priv-agg} is more delicate. It follows the arguments in the proof of Theorem 3.6 of  \citet{bassily2018arxiv} for the most part, which combines the \emph{sparse vector technique} (SVT)  \citep{hardt2010multiplicative} with the \emph{distance to stability} approach from  \citet{thakurta2013differentially}. The only difference in the stated result here is that we used the modern CDP approach to handle the composition which provides tighter constants.

\begin{proof}
	First note that the global sensitivity (Definition \ref{def-global-sen}) of the vote count is $1$. Algorithm~\ref{alg-priv-agg-pate} is a straightforward adaptive composition of $\ell$ Gaussian mechanisms (Lemma \ref{lem-gau}), which satisfies  $\frac{\ell}{2\sigma^2}$-zCDP. By Lemma~\ref{lem:cdp2dp}, we get that the choice of $\sigma$ gives us ($\epsilon,\delta$)-differential privacy.
	
	Let us now address Algorithm~\ref{alg-priv-agg}. First note that $\widehat{\Delta}(x_j)$ as a function of the input dataset $D$ has a global sensitivity of $2$ for all $x_j$, thus $\mathrm{dist}_j$ has a global sensitivity of $1$.
	Following the proof of Theorem 3.6 of  \citet{bassily2018arxiv}, Algorithm~\ref{alg-priv-agg} can be considered a composition of Sparse Vector Technique (SVT) (Algorithm \ref{alg-svt}), which outputs a binary vector of $\{\perp,\top\}$ indicating the failures and successes of passing the screening by SVT, and the distance-to-stability mechanism (Algorithm \ref{alg-dist-ins}) which outputs $\{\hat{h}^{\mathtt{agg}}(x_j)\}$ for all coordinates where the output is $\perp$.
	Check that the length of this binary vector is random and is between $T$ and $\ell$. The  number of $\top$ is smaller than $T$. If $\{\hat{h}^{\mathtt{agg}}(x_j)\}$ is not revealed, then this would be the standard SVT, and the challenge is to add the additional outputs.
	
	The key trick of the proof inspired from the privacy analysis (Lemma \ref{lem-pg-dist-ins}) of the distance-to-stability is to discuss the two cases. In the first case, assume for all $j$ such that the output is $\perp$, $\hat{h}^{\mathtt{agg}}(x_j)$ remains the same over $D,D'$, then adding $\hat{h}^{\mathtt{agg}}(x_j)$ to the output obeys $0$-DP; in the second case, assume that there exists some $j$ where we output $\perp$ such that, $\hat{h}^{\mathtt{agg}}(x_j)$ is different under $D$ and $D'$, then for all these $j$ we know that $\mathrm{dist}_j = 0$ for both $D$ and $D'$. By the choice of $\lambda, w$, we know that the second case happens with probability at most $\delta/2$ using the tail of Laplace distribution and a union bound over all $\ell+T$ independent Laplace random variables. 
	Note that this holds uniformly over all possible adaptive choices of the nature, since this depends only on the added noise.
	
	Conditioning on the event that the second case does not happen, the output of the algorithm is only the binary vector of $\{\perp,\top\}$ from SVT. The SVT with cutoff $T$ is an adaptive composition of $T$ SVTs with cutoff$=1$. By our choice of parameter $\lambda$, each such SVT with cutoff$=1$ obeys pure-DP with privacy parameter $2/\lambda$, hence also satisfy CDP with parameter $2/\lambda^2$ by Proposition~1.4 of  \citet{bun2016concentrated}.
	Composing over $T$ SVTs, we get a CDP parameter of $2T/\lambda^2$. By Proposition~1.3 of  \citet{bun2016concentrated} (Lemma~\ref{lem:cdp2dp}), we can convert CDP to DP. The choice of $\lambda$ is chosen such that the composed mechanism obeys $(\epsilon,\delta/2)$-DP.  Combining with the second case above, this establishes the $(\epsilon,\delta)$-DP of Algorithm~\ref{alg-priv-agg}.
\end{proof}

\begin{theorem}[Restatement of Theorem \ref{thm-ug-bassily}] \label{thm:reproving_bassilyetal}
	Set
	$$T = 3\Big(\E[\mathtt{Err}(\hat{h}_1)] m + \sqrt{\frac{m\log(m/\beta)}{2}}\Big),$$
	$$K = O \Big(\frac{\log(mT / \min(\delta,\beta))\sqrt{T\log(1/\delta)}}{\epsilon}\Big).$$ Let $\hat{h}^S$ be the output of Algorithm~\ref{alg-psq} that uses Algorithm~\ref{alg-priv-agg} for privacy aggregation. With probability at least $1-\beta$ (over the randomness of the algorithm and the randomness of all data points drawn i.i.d.), we have
	$$\mathtt{Err}(\hat{h}^S) \leq  \tilde{O}\Big(\frac{d^2m \log(1/\delta)}{n^2\epsilon^2}  + \sqrt{\frac{d}{m}}\Big)$$
	for the realizable case, and
	$$\mathtt{Err}(\hat{h}^S) \leq  13\mathtt{Err}(h^*) +  \tilde{O}\Big(\frac{m^{1/3}d^{2/3}}{n^{2/3}\epsilon^{2/3}}  + \sqrt{\frac{d}{m}}\Big)$$
	for the agnostic case. 
\end{theorem}

\begin{proof}
The analysis essentially follows the proof of Theorem~\ref{thm-ug-psq} by replacing $h^*$ with $y$. First, by Hoeffding's inequality, with probability $1-\beta$ over the teacher data points, the total number of mistakes made by each teacher classifier is at most $m\E[\mathtt{Err}(\hat{h}_1)] + \sqrt{m\log(m/\beta)/2}$, which is $B$ in Lemma \ref{lem:pigeon_hole}. Then following Lemma \ref{lem:pigeon_hole}, by choose $T=3B=3(m\E[\mathtt{Err}(\hat{h}_1)] + \sqrt{m\log(m/\beta)/2})$, we ensure that the majority voting classifiers are correct and have high margin in at least $m-T$ examples.

\noindent\textbf{In the realizable setting.} Since $\mathtt{Err}(h^*)=0$ and by standard statistical learning theory in the realizable case (Lemma \ref{lem-excess}), for each teacher classifier $\hat{h}_k$ we have
\begin{align*}
\mathtt{Err}(\hat{h}_k) \leq 4\frac{d \log(n/K) + \log(4/\gamma)}{n/K}.
\end{align*}
Substitute our choice of $K=\tilde{O}(\sqrt{T\log(1/\delta)}/\epsilon)$ as in Lemma~\ref{lem:utility_svt} we get that w.h.p.
\begin{align*}
\mathtt{Err}(\hat{h}_k) \leq \tilde{O}\Big(\frac{d\sqrt{T\log(1/\delta)}}{n\epsilon}\Big).
\end{align*}
Plug in the bound into our choice of $T=3(m\E[\mathtt{Err}(\hat{h}_1)] + \sqrt{m\log(m/\beta)/2})$, we get 
\begin{align*}
T\leq \tilde{O}\Big(\frac{dm\sqrt{T\log(1/\delta)}}{n\epsilon}+\sqrt{\frac{m\log(m/\beta)}{2} }\Big).
\end{align*}
By solving the quadratic inequality, we get that $T$ obeys
$$
T \leq \tilde{O}\Big(\frac{d^2m^2\log(1/\delta)}{n^2\epsilon^2} + \sqrt{m} \Big).
$$
Recall that this choice of $K$ and $T$ ensures that Algorithm~\ref{alg-priv-agg} will have at most $T$ unstable queries during the $m$ rounds, which implies that with high probability, the privately released pseudo-labels to those ``stable'' queries are the same as the corresponding true labels.

Now the next technical subtlety is to deal with the dependence in the student learning problem created by the pseudo-labels via a reduction to an ERM learner.  
By the standard Hoeffding-style uniform convergence bound (Lemma \ref{lem-dis-emp}),
\begin{align}
\mathtt{Err}(\hat{h}^S) \leq&\  \widehat{\mathtt{Err}}(\hat{h}^S) + \tilde{O}\Big(\sqrt{\frac{d}{m}}\Big)\nonumber\\
\leq&\ \widehat{\mathtt{Err}}(\hat{h}^{\mathtt{priv}}) + \widehat{\Dis}(\hat{h}^{\mathtt{priv}}, \hat{h}^S) +  \tilde{O}\Big(\sqrt{\frac{d}{m}}\Big)\nonumber\\
\leq&\ 2\widehat{\mathtt{Err}}(\hat{h}^{\mathtt{priv}})  +  \tilde{O}\Big(\sqrt{\frac{d}{m}}\Big)\nonumber\\
\leq&\ \frac{2T}{m} +  \tilde{O}\Big(\sqrt{\frac{d}{m}}\Big)\nonumber\\
=&\  \tilde{O}\Big(\frac{d^2m\log(1/\delta)}{n^2\epsilon^2} + \sqrt{\frac{d}{m}}\Big).\label{eq:thm_bassily_deriv1}
\end{align}
where we applied the triangular inequality in the second line, used that $\hat{h}^S$ is the minimizer of $\widehat{\Dis}(\hat{h}^{\mathtt{priv}}, \cdot)$ in the third line,  and then combined Lemma~\ref{lem:utility_svt} and Lemma~\ref{lem:pigeon_hole} to show that under the appropriate choice of $T$ and $K$ with high probability, $\hat{h}^{\mathtt{priv}}(x_j)$ correctly returns $y_j$ except for up to $T$ example. Finally, the choice of $T$ is substituted.

%

\noindent\textbf{In agnostic setting.} By Lemma \ref{lem-excess}, with high probability, for all teacher classifier $\hat{h}_k$ for $k=1,...,K$, we have
\begin{align*}
\mathtt{Err}(\hat{h}_k) - \mathtt{Err}(h^*) \leq \tilde{O}\Big(\sqrt{\frac{d \log(n/K) + \log(4/\gamma)}{n/K}}\Big).
\end{align*}
Substitute the choice of $K=\tilde{O}(\sqrt{T\log(1/\delta)}/\epsilon)$ from Lemma \ref{lem:utility_svt}, we get
\begin{align*}
\mathtt{Err}(\hat{h}_k) \leq \mathtt{Err}(h^*)+  \tilde{O}\Big(\frac{d^{1/2}T^{1/4}}{n^{1/2}\epsilon^{1/2}}\Big).
\end{align*}
Plug in the above bound into our choice $T=3(m\E[\mathtt{Err}(\hat{h}_1)] + \sqrt{m\log(m/\beta)/2})$, we get that
\begin{equation}\label{eq:thm_bassily_deriv2}
T \leq 3 m \mathtt{Err}(h^*) + \tilde{O}(\sqrt{m}) + \tilde{O}\Big(\frac{m d^{1/2}T^{1/4}}{n^{1/2}\epsilon^{1/2}}\Big).
\end{equation}

Further, we can write
\begin{align}
T &\leq 2(3m \mathtt{Err}(h^*) + \tilde{O}(\sqrt{m})) \cdot \mathbbm{1}\bigg(\tilde{O}\Big(\frac{m d^{1/2}T^{1/4}}{n^{1/2}\epsilon^{1/2}}\Big) \leq \frac{T}{2}\bigg)\nonumber\\
&\quad + \bigg(2 \tilde{O}\Big(\frac{m d^{1/2}}{n^{1/2}\epsilon^{1/2}}\Big)\bigg)^{4/3}\cdot \mathbbm{1}\bigg(\tilde{O}\Big(\frac{m d^{1/2}T^{1/4}}{n^{1/2}\epsilon^{1/2}}\Big) > \frac{T}{2}\bigg)\nonumber\\
&\leq 6m \mathtt{Err}(h^*) + \tilde{O}(\sqrt{m}) + \tilde{O}\Big(\frac{m^{4/3} d^{2/3}}{n^{2/3}\epsilon^{2/3}}\Big),\label{eq:thm_bassily_deriv3}
\end{align}
where the first line talks about two cases of Inequality \eqref{eq:thm_bassily_deriv2}: (1) $T/2 \leq T - \tilde{O}(\frac{m d^{1/2}T^{1/4}}{n^{1/2}\epsilon^{1/2}})\leq 3m \mathtt{Err}(h^*) + \tilde{O}(\sqrt{m})$ if $\tilde{O}(\frac{m d^{1/2}T^{1/4}}{n^{1/2}\epsilon^{1/2}}) \leq T/2$, and (2) $T^{3/4} \leq 2 \tilde{O}(\frac{m d^{1/2}T^{1/4}}{n^{1/2}\epsilon^{1/2}})$ if $\tilde{O}(\frac{m d^{1/2}T^{1/4}}{n^{1/2}\epsilon^{1/2}}) > T/2$; The second line is due to the indicator function is always $\leq 1$.


Similar to the realizable case, now we apply a reduction to ERM. By the Hoeffding's style uniform convergence bound (implied by Lemma \ref{lem-dis-emp})
\begin{align*}
\mathtt{Err}(\hat{h}^S) &\leq \widehat{\mathtt{Err}}(\hat{h}^S) + \tilde{O}\Big(\sqrt{\frac{d}{m}}\Big)\\
&\leq\widehat{\mathtt{Err}}(\hat{h}^{\mathtt{priv}}) + \widehat{\Dis}(\hat{h}^{\mathtt{priv}}, \hat{h}^S)   + \tilde{O}\Big(\sqrt{\frac{d}{m}}\Big)\\
&\leq \widehat{\mathtt{Err}}(\hat{h}^{\mathtt{priv}}) + \widehat{\Dis}(\hat{h}^{\mathtt{priv}}, \hat{h}_1)   + \tilde{O}\Big(\sqrt{\frac{d}{m}}\Big)\\ 
&\leq 2\widehat{\mathtt{Err}}(\hat{h}^{\mathtt{priv}}) + \widehat{\mathtt{Err}}(\hat{h}_1) + \tilde{O}\Big(\sqrt{\frac{d}{m}}\Big)\\ 
&\leq \frac{2T}{m}  + \mathtt{Err}(h^*) + \tilde{O}\Big(\sqrt{\frac{d}{m}}\Big)\\
&\leq 13 \mathtt{Err}(h^*)+\tilde{O}\Big(\frac{m^{1/3}d^{2/3}}{n^{2/3}\epsilon^{2/3}} + \sqrt{\frac{d}{m}}\Big).
\end{align*}
where the second and fourth lines use the triangular inequality of $0-1$ error, the third line uses the fact that $\hat{h}^S$ is the empirical risk minimizer of the student learning problem with labels $\hat{h}^{\mathtt{priv}}$ and the fact that $h_1\in\cH$. The second last line follows from the fact that in those stable queries $\hat{h}^{\mathtt{priv}}(x_j)$ outputs $y_j$, and a standard agnostic learning bound. Finally, in the last line, we obtain 
the stated result by substituting the upper bound of $T$ from  \eqref{eq:thm_bassily_deriv3}.
\end{proof}

The results stated in Table \ref{tab:summary} are obtained by minimizing the bound by choosing a random subset of data points to privately release labels.

\subsection{Learning bound for PATE with Gaussian Mechanism}
In this subsection, we provide a theoretical analysis of the version of PATE from \citet{papernot2017semi,papernot2018scalable} that uses Gaussian mechanism to release the aggregated teacher labels. We will focus on the setting assuming $\tau$-TNC. 
Though this result is not our main contribution, we note that standard PATE is a practical algorithm and this is the first learning-theoretic guarantees of PATE.

\begin{theorem}[Utility guarantee of Algorithm \ref{alg-priv-agg-pate}]\label{thm:standard_pate_tnc}
	Assume the data distribution $\cD$ and the hypothesis class $\cH$ obey the Tsybakov noise condition with parameter $\tau$, then with probability at least $1-\gamma$, there exists universal constant $C$ such that the output $\hat{h}_S$ of  Algorithm~\ref{alg-priv-agg-pate} with parameter $K$ satisfying
	$$K = \frac{6\sqrt{\log(2n)}(\sqrt{m \log(1/\delta)} + \sqrt{m \log(1/\delta) + \epsilon m})}{\epsilon}$$
	obeys that 
	$$
	\mathtt{Err}(\hat{h}^S) - \mathtt{Err}(h^*) \leq \tilde{O}\bigg(\frac{d}{m} + \Big( \frac{d\sqrt{m}}{n\epsilon}\Big)^\frac{\tau}{2-\tau}\bigg).
	$$
	Specifically, in the realizable setting, then it follows that
	$$
	\mathtt{Err}(\hat{h}^S) - \mathtt{Err}(h^*) \leq \tilde{O}\Big(\frac{d}{m} + \frac{d\sqrt{m}}{n\epsilon}\Big).
	$$
\end{theorem}
\begin{proof}
	By the triangular inequality of the $0-1$ error,
	\begin{align}
	\mathtt{Err}(\hat{h}^S) - \mathtt{Err}(h^*)&\leq \Dis(\hat{h}^S,h^*)\nonumber\\
	&\leq \Dis(\hat{h}^S, \tilde{h}^\texttt{priv}) + \Dis(\tilde{h}^\texttt{priv}, h^*)\nonumber\\
	&\leq 2 \Dis(\tilde{h}^\texttt{priv}, h^*) + 2\sqrt{\frac{(d + \log(4/\gamma))\Dis(\tilde{h}^\texttt{priv},h^*)}{m}} + \frac{4(d + \log(4/\gamma))}{m}\nonumber\\
	&\leq 4 \Dis(\tilde{h}^\texttt{priv},h^*) + \tilde{O} \Big(\frac{d}{m}\Big)\label{eq:tnc_pate_deriv1}
	\end{align}
	where the third line follows from the learning bound (Lemma \ref{lem-excess}) with $\tilde{h}^\texttt{priv}$ being the labeling function for the student dataset. The last line is due to $a + 2\sqrt{ab} + b \leq 2a + 2b$ for non-negative $a,b$.
	
	The remaining problem would be finding the upper bound of $\Dis(\tilde{h}^\mathtt{priv},h^*)$. First by Lemma~\ref{lem-tnc}, with probability at least $1-\gamma/2$, $\forall k \in[K]$ we have
	\begin{align*}
	\Dis(\hat{h}_k,h^*) \lesssim \eta^\frac{2}{2-\tau}\Big(\frac{dK\log(n/d)+ \log(2K/\gamma)}{n}\Big)^\frac{\tau}{2-\tau}.
	\end{align*}
	
	Next, conditioning on the teachers, $\tilde{h}^\mathtt{priv}$ is independent for each input and well-defined for all input. Let $Z\sim \cN(0,\sigma^2)$. By Gaussian-tail bound and Markov's inequality,
	\begin{align*}
	&\ \Dis(\tilde{h}^\mathtt{priv},h^*) \\
	\leq&\ \P \Big[|Z| \leq \sigma\sqrt{2\log\big(\frac{2}{\beta}\big)}\Big] \P \bigg[\sum_{k=1}^K \1(\hat{h}_k(x)\neq h^*(x))  \geq \frac{K}{2} - |Z| \bigg| |Z| \leq \sigma\sqrt{2\log\big(\frac{2}{\beta}\big)} \bigg]\\
	&\quad + \P \Big[|Z| > \sigma\sqrt{2\log\big(\frac{2}{\beta}\big)} \Big] \\
	\leq&\ \frac{1}{K/2 - \sigma\sqrt{2\log(2/\beta)}}\sum_{k=1}^K\E[\1(\hat{h}_k(x)\neq h^*(x))] + \beta\\
	\leq&\ \frac{3}{K}  \sum_{k=1}^K \Dis(\hat{h}_k,h^*) +  \frac{1}{n}\\
	\lesssim&\ \eta^\frac{2}{2-\tau}\Big(\frac{dK\log(n/d)+ \log(2K/\gamma)}{n}\Big)^\frac{\tau}{2-\tau}.
	\end{align*}
	In the last line, we choose $\beta = 1/n$ and applied the assumption that $K \geq 6\sigma\sqrt{2\log (2n)}$.
	
	Note that our choice of $\sigma$ satisfies that
	\begin{align*}
	\sqrt{\frac{2m \log(1/\delta)}{\sigma^2}} + \frac{m}{2\sigma^2} = \epsilon.
	\end{align*}
	Solve the equation and we find that
	\begin{align*}
	\sigma = \frac{\sqrt{2 m \log(1/\delta)} + \sqrt{2 m \log(1/\delta) + 2\epsilon m}}{2\epsilon}.
	\end{align*}
	Therefore, the choice of $K$ is
	$$K = \frac{6\sqrt{\log(2n)}(\sqrt{m \log(1/\delta)} + \sqrt{m \log(1/\delta) + \epsilon m})}{\epsilon} = \tilde{O}\Big(\frac{\sqrt{m}}{\epsilon}\Big),
	$$
	where $\epsilon$ is assumed to be small.
	Put everything together, and the excess risk bound is
	\begin{align*}
	\mathtt{Err}(\hat{h}^S) - \mathtt{Err}(h^*) \leq \tilde{O}\bigg(\frac{d}{m} + \Big(\frac{d\sqrt{m}}{n\epsilon}\Big)^\frac{\tau}{2-\tau}\bigg).
	\end{align*}
	
\end{proof}

\begin{remark}
When $m$ is sufficient large ($\frac{d}{m} < \Big(\frac{d\sqrt{m}}{n\epsilon}\Big)^\frac{\tau}{2-\tau}$), it suffices to use a subset of randomly chosen data points to optimize the bound and we obtain an excess risk bound of  $\tilde{O}\Big(\big(\frac{d^{3/2}}{n\epsilon}\big)^\frac{2\tau}{4-\tau} \Big)$. When $\tau=1$, this yields the $\frac{d}{(n\epsilon)^{2/3}}$ rate that matches \citet{bassily2018model}'s analysis of SVT-based PATE.  To avoid any confusions,  Gaussian mechanism-based PATE is still theoretically inferior comparing to SVT-based PATE as we established in Theorem~\ref{thm-ug-psq}.
	\end{remark}

\section{Deferred Proofs of Our Results in Main Paper}\label{sec:deferred_proofs}

In this section, we present full proofs of our results shown in the main paper.

\begin{proposition}[Restatement of Proposition \ref{prop:SVT_works}]
	Assume the learning problem with $n/K$ i.i.d. data points satisfies $(\nu,\xi)$-approximate high-margin condition.  Let Algorithm~\ref{alg-priv-agg} be instantiated with parameters
	$$T \geq  \nu m +  \sqrt{2\nu m \log \Big(\frac{3}{\gamma}\Big)} + \frac{2}{3}\log \Big(\frac{3}{\gamma}\Big)$$
	$$K \geq \max \Big\{ \frac{2\log(3m/\gamma)}{\xi^2}, \frac{3\lambda \big(\log (4m /\delta) +\log(3m/\gamma)\big)}{\xi}\Big\},\footnote{$\lambda = (\sqrt{2T(\epsilon + \log(2/\delta))} + \sqrt{2T\log(2/\delta)})/\epsilon$ according to Algorithm~\ref{alg-priv-agg}.}$$
	then with high probability (over the randomness of the $n$ i.i.d. samples of the private dataset, $m$ i.i.d. samples of the public dataset, and that of the randomized algorithm), Algorithm~\ref{alg-priv-agg} finishes all $m$ rounds and the output is the same as $h^\mathtt{agg}_{\infty}(x_i)$ for all but $T$ of the $i\in[m]$.
\end{proposition}
\begin{proof}
By the Bernstein's inequality, 	with probability $\geq 1-\gamma_2$ over the i.i.d. samples of the public data, the number of queries $j\in[m]$ with $\Delta_{n/K}(x_j)\leq\xi$ is smaller than $\nu m +  \sqrt{2\nu m \log (1/\gamma_2)} + \frac{2}{3}\log(1/\gamma_2)$. $T$ is an upper bound of the above quantity if we choose $\gamma_2 = \gamma/3$.
	
Conditioning on the above event, by Hoeffding's inequality and a union bound, with probability $\geq 1- \gamma_3$ over the i.i.d. samples of the private data (hence the $K$ i.i.d. teacher classifiers), for all $m-T$ queries with $\Delta_{n/K}(x_i)$ larger than $\xi$,  the realized margin (defined in \eqref{eq:margin}) obeys that
\begin{align*}
\widehat{\Delta}(x_j) \geq&\ \E[ \widehat{\Delta}(x_j)  | x_j] -  \sqrt{2K\log \Big(\frac{m}{\gamma_3}\Big)} \\
=&\  2K\Delta_{n/K}(x_i) -  \sqrt{2K\log \Big(\frac{m}{\gamma_3}\Big)} \\
\geq&\ 2K\xi -  \sqrt{2K\log \Big(\frac{m}{\gamma_3}\Big)}.
\end{align*}

It remains to check that under our choice of $T,K$,  $\widehat{\mathrm{dist}}_j > \hat{w} $ for all $j\in[m]$ except the (up to) $T$ exceptions.

	By the tail of Laplace distribution and a union bound, with probability $\geq 1-\gamma_1$, all $m$ Laplace random variables that perturb the distance to stability $\widehat{\mathrm{dist}}_j$ in Algorithm~\ref{alg-dist-ins} is larger than $-2\lambda\log((m+T)/(2\gamma_1))$ and all $T$ Laplace random variables that perturb the threshold $w$ is smaller than $\lambda \log((m+T)/(2\gamma_1))$, where $\lambda$ is chosen according to Algorithm~\ref{alg-priv-agg}. We simplify the above bound by using $T<m$.

It suffices that $K$ is chosen such that
$$
2K\xi -  \sqrt{2K\log \Big(\frac{m}{\gamma_3}\Big)} -  2\lambda\log\Big(\frac{m}{\gamma_1}\Big) > w + \lambda \log\Big(\frac{m}{\gamma_1}\Big).
$$
Substitute Algorithm~\ref{alg-priv-agg}'s choice $w = 3 \lambda \log (2(m+T) /\delta)\leq 3\lambda \log(4m / \delta)$. Assume $K \geq 2\log(m/\gamma_3)/\xi^2$, we have $2K\xi -  \sqrt{2K\log(m/\gamma_3)} \geq K\xi$, thus it suffices that further $K\xi > 3\lambda \big(\log (4m /\delta) +\log(m/\gamma_1) \big)$.

The proof is complete by taking $\gamma_2 = \gamma_3=\gamma/3$ and take union bound over all high probability events described above.
\end{proof}

\begin{theorem}[Restatement of Theorem \ref{thm-agree-infty}]
	Assume the learning problem with $n/K$ i.i.d. data points satisfies $(\nu,\xi)$-approximate high-margin condition and let $K,T$ be chosen according to Proposition~\ref{prop:SVT_works}, furthermore assume that the privacy parameter of choice $\epsilon \leq \log(2/\delta)$, then
	the output classifier $\hat{h}^S$ of Algorithm \ref{alg-psq} in the agnostic setting satisfies that with probability $\geq 1-2\gamma$,
	\begin{align*}
	\mathtt{Err}(\hat{h}^S) - \mathtt{Err}(h^{\mathtt{agg}}_{\infty}) \leq \min_{h\in\cH}\Dis(h,h^{\mathtt{agg}}_{\infty}) + \frac{2T}{m}+ \tilde{O}\Big(\sqrt{\frac{d}{m}}\Big)\leq \min_{h\in\cH}\Dis(h,h^{\mathtt{agg}}_{\infty}) + 2\nu + \tilde{O}\Big(\sqrt{\frac{d}{m}}\Big) .
	\end{align*}
\end{theorem}
\begin{proof}
We follow a similar argument as in the proof of Theorem~\ref{thm-ug-psq}, but replace $h^*$ with $h^{\mathtt{agg}}_{\infty}$. Define $\tilde{h} = \argmin_{h \in \cH} \widehat{\Dis}(h, h^{\mathtt{agg}}_{\infty})$. By the triangular inequality of the $0-1$ error and Lemma~\ref{lem-dis-emp} in Appendix \ref{sec-lemma},
\begin{align}
\mathtt{Err}(\hat{h}^S) - \mathtt{Err}(h^\mathtt{agg}_{\infty})\leq \Dis(\hat{h}^S,h^\mathtt{agg}_{\infty})
\leq \widehat{\Dis}(\hat{h}^S,h^\mathtt{agg}_{\infty}) + \tilde{O}\Big(\sqrt{\frac{d}{m}}\Big).\label{eq:agn-derive1}
\end{align}
By the triangular inequality, we have $\widehat{\Dis}(\hat{h}^S,h^\mathtt{agg}_{\infty}) \leq \widehat{\Dis}(\hat{h}^S,\hat{h}^\mathtt{priv}) + \widehat{\Dis}(\hat{h}^\mathtt{priv},h^\mathtt{agg}_{\infty})$, therefore,
\begin{align*}
\eqref{eq:agn-derive1} &\leq \widehat{\Dis}(\hat{h}^S,\hat{h}^\mathtt{priv}) + \widehat{\Dis}(\hat{h}^\mathtt{priv},h^\mathtt{agg}_{\infty}) + \tilde{O}\Big(\sqrt{\frac{d}{m}}\Big)\\
&\leq \widehat{\Dis}(\tilde{h},\hat{h}^\mathtt{priv}) + \widehat{\Dis}(\hat{h}^\mathtt{priv},h^\mathtt{agg}_{\infty}) + \tilde{O}\Big(\sqrt{\frac{d}{m}}\Big)\\
&\leq \widehat{\Dis}(\tilde{h},{h}^\mathtt{agg}_
\infty) + 2\widehat{\Dis}(\hat{h}^\mathtt{priv},h^\mathtt{agg}_{\infty}) + \tilde{O}\Big(\sqrt{\frac{d}{m}}\Big)\\
&\leq \min_{h\in\cH}\Dis(h,h^{\mathtt{agg}}_{\infty}) + 2\widehat{\Dis}(\hat{h}^\mathtt{priv},h^\mathtt{agg}_{\infty}) + \tilde{O}\Big(\sqrt{\frac{d}{m}}\Big).
\end{align*}
In the second line, we applied the fact that $\hat{h}^S = \argmin_{h \in \cH} \widehat{\Dis}(h, \hat{h}^{\mathtt{priv}})$; in the third line, we applied triangular inequality again and the last line is true because $\tilde{h} = \argmin_{h \in \cH} \widehat{\Dis}(h, h^{\mathtt{agg}}_{\infty})$.

Recall that $T$ is the unstable cutoff in Algorithm \ref{alg-psq}. The proof completes by invoking Proposition~\ref{prop:SVT_works} which implies that $\widehat{\Dis}(\hat{h}^\mathtt{priv},h^\mathtt{agg}_\infty) \leq T/m$ with high probability.

\end{proof}

\begin{theorem}[Restatement of Theorem \ref{thm-vote-space}]
	Under the same assumption of Theorem~\ref{thm-agree-infty}, suppose we train an ensemble classifier within the voting hypothesis space $\mathrm{Vote}_K(\cH)$ in the student domain, then the output classifier $\hat{h}^S$ of Algorithm \ref{alg-psq} in the agnostic setting satisfies that with probability $\geq 1-2\gamma$,
\begin{align*}
	\mathtt{Err}(\hat{h}^S) - \mathtt{Err}(h^{\mathtt{agg}}_{\infty}) \leq \frac{4T}{m} + \frac{5(Kd + \log(4/\gamma))}{m}= \tilde{O}\Big( \nu + \frac{\log(4/\gamma)}{m} + \frac{d \sqrt{\nu}}{\xi \sqrt{m}}\Big).
\end{align*}
%
\end{theorem}

\begin{proof}
Define $\hat{h}^S = \argmin_{h \in \mathrm{Vote}_K(\cH)} \widehat{\Dis}(h, \hat{h}^{\mathtt{priv}})$ and $\tilde{h} = \argmin_{h \in \mathrm{Vote}_K(\cH)} \widehat{\Dis}(h, h^{\mathtt{agg}}_{\infty})$. By the triangular inequality of the $0-1$ error,
\begin{align}
\mathtt{Err}(\hat{h}^S) - \mathtt{Err}(h^\mathtt{agg}_{\infty}) &\leq
\Dis(\hat{h}^S, h^\mathtt{agg}_\infty)\nonumber\\
&\leq \widehat{\Dis}(\hat{h}^S,h^\mathtt{agg}_{\infty}) + 2\sqrt{\frac{(Kd + \log(4/\gamma))\widehat{\Dis}(\hat{h}^S, h^\mathtt{agg}_{\infty})}{m}} + \frac{4(Kd + \log(4/\gamma))}{m}\nonumber\\
&\leq 2\widehat{\Dis}(\hat{h}^S,h^\mathtt{agg}_{\infty}) + \frac{5(Kd + \log(4/\gamma))}{m},\label{eq:agn_thm_deriv1}
\end{align}
where the second line follows from the first statement of Lemma \ref{lem-dis-emp} in Appendix \ref{sec-lemma} with $z = h^\mathtt{agg}_{\infty}(x)$ and the third line is due to $a + 2\sqrt{ab} + b \leq 2a + 2b$ for non-negative $a, b$.

By the triangular inequality, we have $\widehat{\Dis}(\hat{h}^S,h^\mathtt{agg}_{\infty}) \leq \widehat{\Dis}(\hat{h}^S,\hat{h}^\mathtt{priv}) + \widehat{\Dis}(\hat{h}^\mathtt{priv},h^\mathtt{agg}_{\infty})$, therefore,
\begin{align*}
\eqref{eq:agn_thm_deriv1} & \leq 2\widehat{\Dis}(\hat{h}^S,\hat{h}^\mathtt{priv}) + 2\widehat{\Dis}(\hat{h}^\mathtt{priv},h^\mathtt{agg}_{\infty}) + \frac{5(Kd + \log(4/\gamma))}{m}\\
& \leq 2\widehat{\Dis}(\tilde{h},\hat{h}^\mathtt{priv}) + 2\widehat{\Dis}(\hat{h}^\mathtt{priv},h^\mathtt{agg}_{\infty}) + \frac{5(Kd + \log(4/\gamma))}{m}\\
& \leq 2\widehat{\Dis}(\tilde{h},h^\mathtt{agg}_{\infty}) + 4\widehat{\Dis}(\hat{h}^\mathtt{priv},h^\mathtt{agg}_{\infty}) + \frac{5(Kd + \log(4/\gamma))}{m}\\
& \leq 4\widehat{\Dis}(\hat{h}^\mathtt{priv},h^\mathtt{agg}_{\infty}) + \frac{5(Kd + \log(4/\gamma))}{m}.
\end{align*}
In the second line, we applied the fact that $\hat{h}^S = \argmin_{h \in \mathrm{Vote}_K(\cH)} \widehat{\Dis}(h, \hat{h}^{\mathtt{priv}})$; in the third line, we applied triangular inequality again and the last line is true because $\widehat{\Dis}(\tilde{h},h^\mathtt{agg}_{\infty}) = 0$ since $\tilde{h}$ is the minimizer and that $h^\mathtt{agg}_{\infty} \in \mathrm{Vote}_K(\cH)$.

Recall that $T$ is the unstable cutoff in Algorithm \ref{alg-psq}. The proof completes by using that $\widehat{\Dis}(\hat{h}^\mathtt{priv},h^\mathtt{agg}_\infty) \leq T/m$ with probability $1-\gamma$ according to Proposition~\ref{prop:SVT_works} and substitute the choices of $T$ and $K$ accordingly.
\end{proof}

\begin{lemma}\label{lem-aa-err}
If the disagreement-based agnostic active learning algorithm is given a stream of $m$ unlabeled data points, then with probability at least $1-\gamma$, the algorithm returns a hypothesis $h$ obeying that,
\begin{align*}
\mathtt{Err}(h) - \mathtt{Err}(h^*) &\lesssim \frac{d \log (\theta(d/m)) + \log (1/\gamma)}{m}+ \sqrt{\frac{\mathtt{Err}(h^*)( d \log (\theta(\mathtt{Err}(h^*)) + \log (1/\gamma))}{m}}.
\end{align*}
\end{lemma}
\begin{proof}
From Lemma 3.1 of \citet{hanneke2014theory}, we learn that for any hypothesis $h$ survive in version space $V$ must satisfy
\begin{align*}
\mathtt{Err}(h) - \mathtt{Err}(h^*) \leq 2 U(m, \gamma).
\end{align*}
Then by the definition of $U(m, \gamma)$ shown in Algorithm \ref{alg-active_learning}, we have
\begin{align*}
\mathtt{Err}(h) - \mathtt{Err}(h^*) &\lesssim \frac{d \log (\theta(d/m)) + \log (1/\gamma)}{m}+ \sqrt{\frac{\mathtt{Err}(h^*)( d \log (\theta(\mathtt{Err}(h^*)) + \log (1/\gamma))}{m}}.
\end{align*}
\end{proof}

\begin{theorem}[Restatement of Theorem \ref{thm-ug-asq}]
With probability at least $1-\gamma$, there exists universal constants $C_1,C_2$ such that for all
\begin{align*}
\alpha \geq C_1\max \bigg\{\eta^\frac{2}{2-\tau}\Big(\frac{dK\log(n/d)+ \log(2K/\gamma)}{n}\Big)^\frac{\tau}{2-\tau}, \frac{d\log((m+n)/d) + \log(2/\gamma)}{m}\bigg\},
\end{align*}
the output $\hat{h}^S$ of  Algorithm~\ref{alg-asq} with parameter $\ell,K$ satisfying
\begin{align*}
\ell = C_2\theta(\alpha)\Big( 1 + \log \big(\frac{1}{\alpha}\big) \Big) \bigg( d\log(\theta(\alpha)) + \log \Big(\frac{\log(1/\alpha)}{\gamma/2} \Big)\bigg)
\end{align*}
$$K = \frac{6\sqrt{\log(2n)}(\sqrt{\ell \log(1/\delta)} + \sqrt{\ell \log(1/\delta) + \epsilon \ell})}{\epsilon}$$
obeys that 
$$
\mathtt{Err}(\hat{h}^S) - \mathtt{Err}(h^*) \leq \alpha.
$$
Specifically, when we choose
\begin{align*}
\alpha= C_1\max \bigg\{\eta^\frac{2}{2-\tau}\Big(\frac{dK\log(n/d)+ \log(2K/\gamma)}{n}\Big)^\frac{\tau}{2-\tau}, \frac{d\log((m+n)/d) + \log(2/\gamma)}{m}\bigg\},  
\end{align*}
and also $\epsilon\leq \log(1/\delta)$, then it follows that
\begin{align*}
\mathtt{Err}(\hat{h}^S) - \mathtt{Err}(h^*) = \tilde{O}\bigg(\max\Big\{\big(\frac{d^{1.5}\sqrt{\theta(\alpha)\log(1/\delta)}}{n\epsilon}\big)^{\frac{\tau}{2-\tau}}, \frac{d}{m}\Big\}\bigg),
\end{align*}
where $\tilde{O}$ hides logarithmic factors in $m,n,1/\gamma$.
\end{theorem}
\begin{proof}
\textbf{Step 1: Teachers are good.}
By Lemma~\ref{lem-tnc}, with probability at least $1-\gamma/2$, $\forall k \in[K]$ we have
\begin{align*}
\Dis(\hat{h}_k,h^*) \lesssim \eta^\frac{2}{2-\tau}\Big(\frac{dK\log(n/d)+ \log(2K/\gamma)}{n}\Big)^\frac{\tau}{2-\tau}.
\end{align*}

\textbf{Step 2: PATE is just as good.}
Let $\tilde{h}^\mathtt{priv}$ be a randomized classifier from Line 4 of Algorithm~\ref{alg-priv-agg-pate}.  Conditioning on the teachers, this classifier is independent for each input and well-defined for all input. Note that $\hat{h}^\mathtt{priv}$ that uses Algorithm~\ref{alg-priv-agg} do not have these properties.
Let $Z\sim \cN(0,\sigma^2)$. By Gaussian-tail bound and Markov's inequality,
\begin{align*}
&\ \Dis(\tilde{h}^\mathtt{priv},h^*) \\
\leq&\ \P \Big[|Z| \leq \sigma\sqrt{2\log\big(\frac{2}{\beta}\big)}\Big] \P \bigg[\sum_{k=1}^K \1(\hat{h}_k(x)\neq h^*(x))  \geq \frac{K}{2} - |Z| \bigg| |Z| \leq \sigma\sqrt{2\log\big(\frac{2}{\beta}\big)} \bigg]\\
&\quad + \P \Big[|Z| > \sigma\sqrt{2\log\big(\frac{2}{\beta}\big)} \Big] \\
\leq&\ \frac{1}{K/2 - \sigma\sqrt{2\log(2/\beta)}}\sum_{k=1}^K\E[\1(\hat{h}_k(x)\neq h^*(x))] + \beta\\
\leq&\ \frac{3}{K}  \sum_{k=1}^K \Dis(\hat{h}_k,h^*) +  \frac{1}{n}\\
\lesssim&\ \eta^\frac{2}{2-\tau}\Big(\frac{dK\log(n/d)+ \log(2K/\gamma)}{n}\Big)^\frac{\tau}{2-\tau}.
\end{align*}
In the last line, we choose $\beta = 1/n$ and applied the assumption that $K \geq 6\sigma\sqrt{2\log (2n)}$.

\textbf{Step 3: Oracle reduction to active learning bounds.}
Note that $\tilde{h}^\mathtt{priv}$ is the labeling function in the student learning problem. So the above implies that the student learning problem is close to realizable:
$$\min_{h\in\cH}\Dis(\tilde{h}^\mathtt{priv},h) \leq  \Dis(\tilde{h}^\mathtt{priv},h^*) \lesssim \eta^\frac{2}{2-\tau}\Big(\frac{dK\log(n/d)+ \log(2K/\gamma)}{n}\Big)^\frac{\tau}{2-\tau}.$$

By the above, and the agnostic active learning bounds in Lemma~\ref{lem-a2l}, to achieve an excess risk bound of $\alpha\geq \Dis(\tilde{h}^\mathtt{priv},h^*) := \mathtt{Err}^*$ in the student learning problem with probability at least $1-\gamma/2$, with unbounded $m$, it suffices to choose $\ell$ to be 
\begin{align*}
    &\ C \theta(\mathtt{Err}^*+\alpha)\bigg( \frac{(\mathtt{Err}^*)^2}{\alpha^2} + \log\Big(\frac{1}{\alpha}\Big) \bigg) \bigg( d\log(\theta(\mathtt{Err}^*+\alpha)) + \log\Big(\frac{\log(1/\alpha)}{\gamma}\Big)\bigg)\\
    \leq &\ C \theta(\alpha) (1+\log(1/\alpha)) \bigg( d\log(\theta(\alpha)) + \log\Big(\frac{\log(1/\alpha)}{\gamma}\Big)\bigg).
\end{align*}
This implies an error bound of 
\begin{align*}
\Dis(\hat{h}_S,\tilde{h}^\mathtt{priv})&\leq \min_{h\in\cH}\Dis(\tilde{h}^\mathtt{priv},h)  +  \alpha \\
&\leq \Dis(\tilde{h}^\mathtt{priv},h^*)  + \alpha \leq 2\alpha.
\end{align*}
When $m$ is small, we might not have enough data points to obtain $\alpha = O(\Dis(\tilde{h}^\mathtt{priv},h^*))$ in this case the error is dominated by our bounds in Lemma~\ref{lem-aa-err}, which says that we can take $$\alpha = C\max\Big\{ \mathtt{Err}^*, \frac{d\log(m/d) + \log(2/\gamma)}{m}\Big\}.$$

\textbf{Step 4 Putting everything together.}
\begin{align*}
\mathtt{Err}(\hat{h}^S) - \mathtt{Err}(h^*) &\leq \Dis(\hat{h}^S, \tilde{h}^\mathtt{priv}) + \Dis(\tilde{h}^\mathtt{priv}, h^*)\\
&\lesssim \Dis(\tilde{h}^\mathtt{priv},h^*)  + \alpha\\
& \lesssim \eta^\frac{2}{2-\tau}\Big(\frac{dK\log(n/d)+ \log(2K/\gamma)}{n}\Big)^\frac{\tau}{2-\tau} + \alpha.
\end{align*}
The proof is complete by substituting our choice of $K = 6\sigma\sqrt{2\log(2n)}$, and furthermore by the standard privacy calibration of the Gaussian mechanism, our choice of $\sigma$ satisfies that
\begin{align*}
\sqrt{\frac{2\ell \log(1/\delta)}{\sigma^2}} + \frac{\ell}{2\sigma^2} = \epsilon.
\end{align*}
following the specification of Algorithm \ref{alg-priv-agg-pate}.
Solve the equation and we find that
\begin{align*}
\sigma = \frac{\sqrt{2 \ell \log(1/\delta)} + \sqrt{2 \ell \log(1/\delta) + 2\epsilon \ell}}{2\epsilon},
\end{align*}
where $\epsilon$ is assumed to be small.
\end{proof}
\section{Technical Lemmas from Statistical learning Theory}\label{sec-lemma}

In this section, we cite a few results from statistical learning theory that we will use as part of our analysis.

\begin{lemma}[Pointwise convergence \citep{bousquet2004introduction}]\label{lem-gen-bern}
Let $(x,z)$ be drawn from any distribution $\cD$ supported on $\cX\times \cY$. Let $\Dis$ and $\widehat{\Dis}$ be the expected and empirical disagreement evaluated on $n$ i.i.d. samples from $\cD$. For each fixed $h \in \cH$, the following generalization error bound holds with probability $1-\gamma$,
\begin{align*}
\Dis(h,z) \leq \widehat{\Dis}(h,z) + \sqrt{\frac{2\Dis(h,z)\log(1/\gamma)}{n}} + \frac{2\log(1/\gamma)}{3n},
\end{align*}
where $n$ is the number of data points.
\end{lemma}
This is a standard application of the Bernstein's inequality.

\begin{lemma}[Uniform convergence \citep{bousquet2004introduction}]\label{lem-dis-emp}
Under the same conditions of Lemma~\ref{lem-gen-bern}, and in addition assume that  $d$ is the VC-dimension of $\cH$,
Then with probability at least $1-\gamma$, $\forall h \in \cH$ simultaneously,
\begin{align*}
\Dis(h,z) - \widehat{\Dis}(h,z) \leq 2\sqrt{\frac{(d + \log(4/\gamma)) \widehat{\Dis}(h,z)}{n}} + \frac{4(d + \log(4/\gamma))}{n}.
\end{align*}
and
\begin{align*}
\Dis(h,z) - \widehat{\Dis}(h,z) \leq 2\sqrt{\frac{(d + \log(4/\gamma)) \Dis(h,z)}{n}} + \frac{4(d + \log(4/\gamma))}{n}.
\end{align*}
\end{lemma}
The above lemma is simply the uniform Bernstein's inequality over a hypothesis class with VC-dimension $d$.
We will be taking $z$ to be $h^*$ in the cases when we work with noise conditions and $h^{\mathtt{agg}}_{\infty}(x)$ in the agnostic case.

\begin{lemma}[Learning bound \citep{bousquet2004introduction}]\label{lem-excess}
Let $d$ be the VC-dimension of $\cH$, the excess risk is bounded with probability $1-\gamma$,
\begin{align*}
\mathtt{Err}(\hat{h}) \leq \mathtt{Err}(h^*) + 2\sqrt{\mathtt{Err}(h^*) \frac{d \log(n) + \log (4/\gamma)}{n}} + 4 \frac{d \log(n) + \log(4/\gamma)}{n},
\end{align*}
where $n$ is the number of data points we sample.

\end{lemma}

\begin{lemma}[Passive learning bound under TNC (Lemma 3.4 of \citet{hanneke2014theory})]\label{lem-tnc-excess}
Let $d$ be the VC-dimension of the class $\cH$.
Assume Tsybakov noise condition with parameters $ \tau$, the excess risk is bounded with probability $1-\gamma$,
\begin{align*}
\mathtt{Err}(\hat{h}) - \mathtt{Err}(h^*)  \lesssim \bigg(\frac{1}{n} \Big(d \log\big(\frac{n}{d}\big) + \log\big(\frac{1}{\gamma}\big) \Big) \bigg)^\frac{1}{2-\tau},
\end{align*}
where $n$ is the number of data points.
\end{lemma}

Finally, we need a result from active learning. 

\begin{lemma}[Agnostic active learning bound (Theorem 5.4 of \citet{hanneke2014theory})]\label{lem-a2l}
Let $\cH$ be a class with VC-dimension $d$.
With probability at least $1-\gamma$, there is a universal constant $C$, such that the agnostic active learning algorithm (see Algorithm~\ref{alg-active_learning}) outputs a classifier with an access risk of $\alpha$ with 
$$C \theta(\mathtt{Err}^*+\alpha)\left( \frac{(\mathtt{Err}^*)^2}{\alpha^2} + \log \Big(\frac{1}{\alpha}\Big) \right) \left( d\log(\theta(\mathtt{Err}^*+\alpha)) + \log \Big(\frac{\log(1/\alpha)}{\gamma}\Big)\right),$$
where $\mathtt{Err}^* = \argmin_{h\in \cH} \mathtt{Err}(h)$.
\end{lemma}
\section{Additional Information about Differential Privacy}\label{sec-info}

In this section, we cite a few results from differential privacy that we will use as part of our analysis.

\begin{lemma}[Post-processing  \citep{dwork2006calibrating}]\label{lem-post-pro}
	If a randomized algorithm $\cM: \cZ^* \rightarrow \mathcal{R}$ is ($\epsilon, \delta$)-DP, then for any function $f:\mathcal{R} \rightarrow \mathcal{R'}$, $f \circ \cM$ is also ($\epsilon, \delta$)-DP.
\end{lemma}

\begin{definition}[Global sensitivity  \citep{dwork2014algorithmic}]\label{def-global-sen}
	A function $f: \cZ^* \rightarrow \mathcal{R}$ has global sensitivity $\vartheta$ if
	\begin{align*}
	\max_{|D - D'|=1} \| f(D) - f(D')\|_1 = \vartheta.
	\end{align*}
\end{definition}

\begin{lemma}[Laplace mechanism  \citep{dwork2006calibrating}]\label{lem-lap}
	If a function $f:\cZ^n \rightarrow \mathcal{R}^p$ has global sensitivity $\vartheta$, then the randomized algorithm $\cM$, which on input $D$ outputs $f(D) + b$, where $b \sim \mathtt{Lap}(\vartheta/\epsilon)^p$, satisfies $\epsilon$-DP. The $\mathtt{Lap}(\lambda)^p$ denotes a vector of $p$ i.i.d. samples from the Laplace distribution $\mathtt{Lap}(\lambda)$.
\end{lemma}

\begin{definition}[$\ell_2$-sensitivity  \citep{dwork2014algorithmic}]\label{def-ell2-sen}
	A function $f: \cZ \rightarrow \mathcal{R}$ has $\ell_2$ sensitivity $\vartheta_2$ if
	\begin{align*}
	\max_{|D - D'|=1} \| f(D) - f(D')\|_2 = \vartheta_2.
	\end{align*}
\end{definition}

\begin{lemma}[Gaussian mechanism  \citep{dwork2014algorithmic}]\label{lem-gau}
	If a function $f:\cZ^n \rightarrow \mathcal{R}^p$ has $\ell_2$-sensitivity $\vartheta_2$, then the randomized algorithm $\cM$, which on input $D$ outputs $f(D) + b$, where $b \sim \cN(0, \sigma^2)^p$, satisfies ($\epsilon,\delta$)-DP, where $\sigma \geq c \vartheta_2 / \epsilon$ and $c^2 > 2 \log(1.25/\delta)$. The $\cN(0, \sigma^2)^p$ denotes a vector of $p$ i.i.d. samples from the Gaussian distribution $\cN(0, \sigma^2)$.
\end{lemma}

\begin{algorithm}[!htbp]
	\caption{Sparse Vector Technique  \citep{dwork2010boosting,dwork2014algorithmic}} 
	\label{alg-svt}
	{\bf Input:}
	Dataset $D$, query set $\mathcal{Q} = \{q_1, ...,q_m\}$, privacy parameters $\epsilon, \delta >0$, unstable query cutoff $T$, threshold $w$.
	\begin{algorithmic}[1]
		\STATE $c \leftarrow 0, \lambda \leftarrow \sqrt{32T \log(1/\delta)/\epsilon}, \hat{w} \leftarrow w + \mathtt{Lap}(\lambda)$.
		\FOR{$q \in \mathcal{Q}$ and $c \leq T$}
		\STATE $\hat{q} \leftarrow q + \mathtt{Lap}(2\lambda)$.
		\IF{$\hat{q} > \hat{w}$}
		\STATE Output $\top$.
		\ELSE
		\STATE Output $\perp$. $\hat{w} \leftarrow w + 1, c \leftarrow c + 1$.
		\ENDIF
		\ENDFOR
	\end{algorithmic}
\end{algorithm}

\begin{lemma}[Privacy guarantee of Algorithm \ref{alg-svt}  \citep{dwork2014algorithmic}]\label{lem-pg-svt}
	Algorithm \ref{alg-svt} is ($\epsilon, \delta$)-DP.
\end{lemma}

\begin{lemma}[Utility guarantee of Algorithm \ref{alg-svt}  \citep{dwork2014algorithmic}]\label{lem-ug-svt}
	For $\phi = \log(2mT/\beta)\sqrt{512T \log(1/\delta)}/\epsilon$, and any set of $m$ queries, define the set $L(\phi) = \{i: q_i(D) \leq w + \phi\}$. If $|L(\phi)| \leq T$, then w.p. at least $1-\beta: \forall i \notin L(\phi)$ Algorithm \ref{alg-svt} outputs $\top$.
\end{lemma}

\begin{definition}[$k$-stability  \citep{thakurta2013differentially}]\label{def-k-sta}
	A function $f: \cZ \rightarrow \mathcal{R}$ is $k$ stable on dataset $D$ if adding or removing any $k$ elements from D does not change the value of $f$, i.e., $f(D) = f(D')$ for all $D'$ such that $|D- D'| \leq k$. We say $f$ is stable on $D$ if it is (at least) $1$-stable on $D$, and unstable otherwise. 
\end{definition}

\begin{algorithm}[!htbp]
	\caption{Distance to Instability Framework  \citep{thakurta2013differentially}} 
	\label{alg-dist-ins}
	{\bf Input:}
	Dataset $D$, function $f: \cZ \rightarrow \mathcal{R}$, distance to instability $\mathrm{dist}_f: \cZ \rightarrow \mathcal{R}$, thereshold $\Gamma$, privacy parameter $\epsilon >0$.
	\begin{algorithmic}[1]
		\STATE $\widehat{\mathrm{dist}} \leftarrow \widehat{\mathrm{dist}}_f(D) + \mathtt{Lap}(1/\epsilon)$.
		\IF{$\widehat{\mathrm{dist}} > \Gamma$}
		\STATE Output $f(D)$.
		\ELSE
		\STATE Output $\perp$.
		\ENDIF
	\end{algorithmic}
\end{algorithm}

\begin{lemma}[Privacy guarantee of Algorithm \ref{alg-dist-ins}  \citep{bassily2018model}]\label{lem-pg-dist-ins}
	If the threshold $\Gamma = \log (1/\delta)/\epsilon$, and the distance to instability function $\mathrm{dist}_f(D) = \argmax_k (f(D)$ is $k$-stable), then Algorithm \ref{alg-dist-ins} is ($\epsilon, \delta$)-DP.
\end{lemma}

\begin{lemma}[Utility guarantee of Algorithm \ref{alg-dist-ins}  \citep{thakurta2013differentially}]\label{lem-ug-dist-ins}
	If the threshold $\Gamma = \log (1/\delta)/\epsilon$, and the distance to instability function $\mathrm{dist}_f(D) = \argmax_k (f(D)$ is $k$-stable), and $f(D)$ is ($(\log(1/\delta) + \log(1/\beta))/\epsilon$)-stable, then Algorithm \ref{alg-dist-ins} outputs $f(D)$ w.p. at least $1-\beta$.
\end{lemma}

\begin{definition}[Definition 1.1 of  \citep{bun2016concentrated}]
    $\cM$ obeys $(\xi,\rho)$-zCDP if for two adjacent dataset $D,D'$, for all $\phi\in(1,\infty)$, the Renyi-divergence of order $\phi$ below obeys that
    $$
    D_{\phi}(\cM(D)\|\cM(D')) \leq \xi + \rho\alpha.
    $$
    When $\xi=0$, we also call it $\rho$-zCDP (or simply $\rho$-CDP, since we are not considering other versions of CDPs in this paper). 
\end{definition}

The following two lemmas will be used in the privacy analysis of the SVT-based PATE.
\begin{lemma}[Proposition~1.3 of \citep{bun2016concentrated}]\label{lem:cdp2dp}
	If $\cM$ obeys $\rho$-zCDP, then $\cM$ is $(\rho + 2\sqrt{\rho\log(1/\delta)},\delta)$-DP for any $\delta > 0$.
\end{lemma}

\begin{lemma}[Proposition~1.4 of \citep{bun2016concentrated}]\label{lem:dp2cdp}
	If $\cM$ obeys $\epsilon$-DP, then $\cM$ obeys $\frac{\epsilon^2}{2}$-CDP.
\end{lemma}

\section{Simulation with adult dataset}\label{sec-adult-dataset}
In this section, we empirically estimate the expected margin $\Delta(x) = |\mathbb{E}[\hat{h}_1(x)|x]-0.5|$ and  $|\mathbb{E}[\mathbbm{I}(h(x) \neq h^*(x)) |x]-0.5|$ on the Adult dataset. Note that in $\Delta(x)$, we do not require the teachers to agree on $y$ or $h^*$ but measure the extent to which they agree with $\hat{h}^\mathtt{agg}$. In the latter one, we measure the degree of agreement between teachers and $h^*$.

The UCI Adult dataset is also known as ``Census Income" dataset, which is used to predict whether an individual's annual income exceeds \$$50,000$. We partition the original training set as the private dataset and the testing set as the public dataset. To simulate the PATE setting, we train 250 logistic regression models on the private domain and use this ensemble to answer $2,000$ queries from the public domain. Note that under this setting, the private domain contains $36,631$ records and the public domain has $10,211$ unlabelled records. We train $h^*$ with the entire private dataset using the logistic regression model.

Results shown in Figure~\ref{fig:main_fig} demonstrate that even though we do not know the distribution of the data at all, the teacher ensemble agrees on the large majority of the examples. Moreover, when they agree, they agree on $h^*$ in most cases and only in very rare cases when they agree on the wrong answers with high-margin. To say it differently, our assumption on the Tsybakov noise condition could be a good approximation to the real-life datasets. In Figure \ref{fig:corre}, we plot the correlations between $\Delta(x)$ and $\E[\1[h(x) \neq h^*(x)] | x] - 0.5$ over $200$ queries. The $x$-axis is the cumulation of  $\Delta(x)$ and the $y$-axis is the cumulation of $\E[\1[h(x) \neq h^*(x)] | x] - 0.5$. There is a nearly perfect linear line in the figure, which indicates they are highly correlated and majority voting tends to agree on $h^*$ in most cases on the Adult dataset.

\begin{figure}[!htbp]
    \centering
    \begin{minipage}{0.49\linewidth}\centering
		\includegraphics[width=\textwidth]{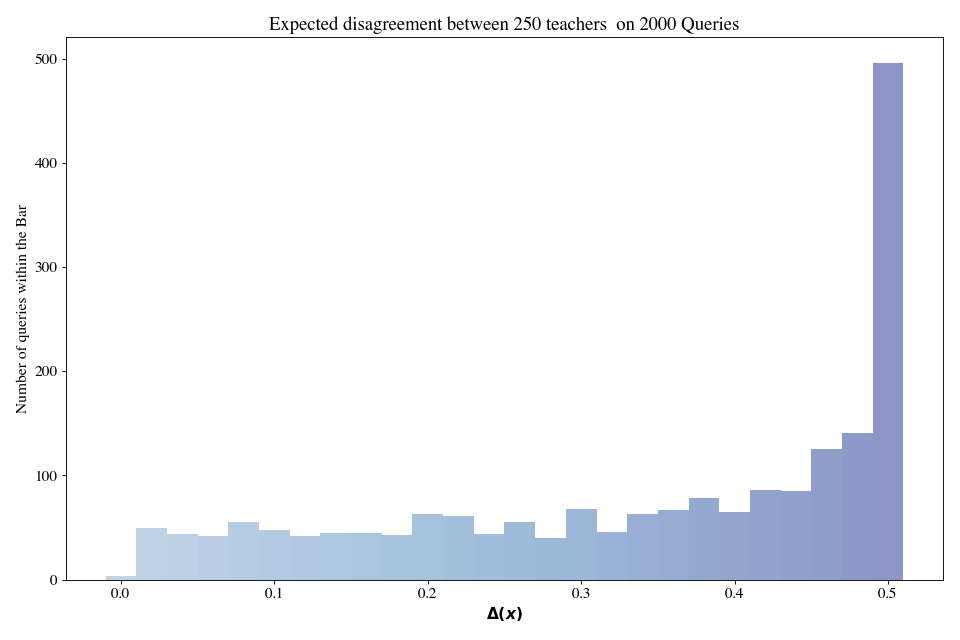}
	\end{minipage}
	\begin{minipage}{0.49\linewidth}\centering
		\includegraphics[width=\textwidth]{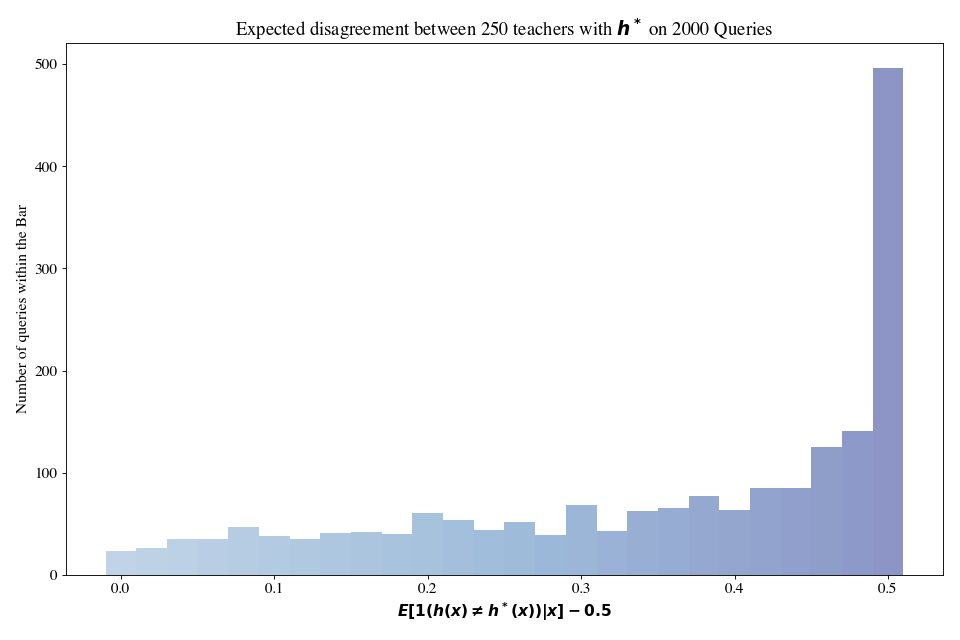}
	\end{minipage}
	\caption{Empirical distribution of the margins on the Adult dataset.}
	\label{fig:main_fig}
\end{figure}

\begin{figure}[!htbp]
    \centering
    \begin{minipage}{0.49\linewidth}\centering
		\includegraphics[width=\textwidth]{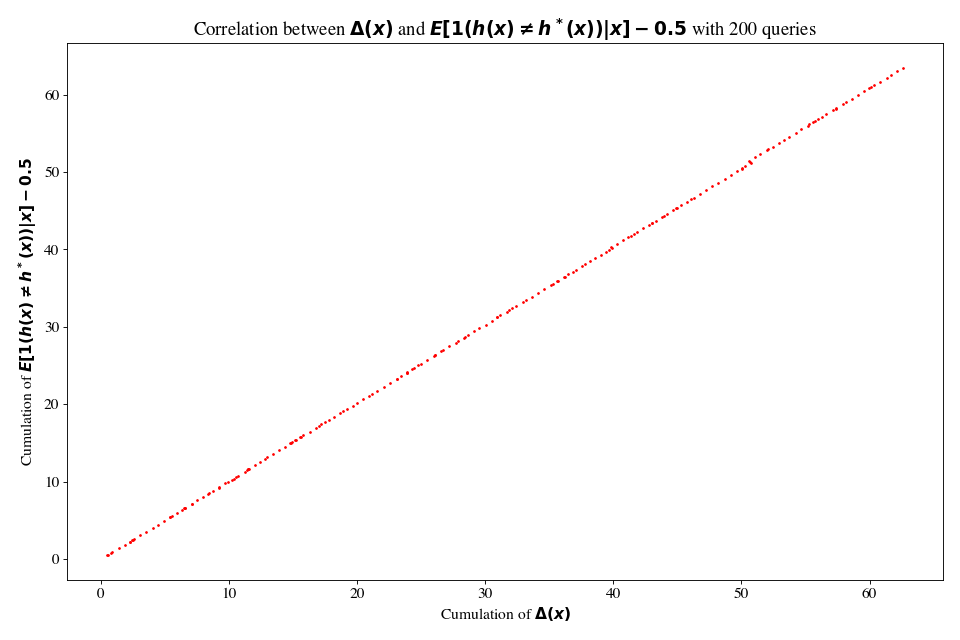}
	\end{minipage}
	\caption{Correlations between margins on the Adult dataset.}
	\label{fig:corre}
\end{figure}

\end{document}